\documentclass{article}
\usepackage{microtype}
\usepackage{fullpage}
\usepackage{graphicx}
\usepackage{hyperref}

\usepackage{amsmath,amsfonts,bm}

\def\1{\bm{1}}

\def\vh{{\bm{h}}}

\def\vs{{\bm{s}}}

\def\vu{{\bm{u}}}
\def\vv{{\bm{v}}}

\def\mD{{\bm{D}}}

\def\mI{{\bm{I}}}

\def\mN{{\bm{N}}}

\def\mU{{\bm{U}}}
\def\mV{{\bm{V}}}
\def\mW{{\bm{W}}}
\def\mX{{\bm{X}}}
\def\mY{{\bm{Y}}}
\def\mZ{{\bm{Z}}}

\DeclareMathAlphabet{\mathsfit}{\encodingdefault}{\sfdefault}{m}{sl}
\SetMathAlphabet{\mathsfit}{bold}{\encodingdefault}{\sfdefault}{bx}{n}

\def\gA{{\mathcal{A}}}

\def\gG{{\mathcal{G}}}
\def\gH{{\mathcal{H}}}

\def\gM{{\mathcal{M}}}
\def\gN{{\mathcal{N}}}
\def\gO{{\mathcal{O}}}

\def\gR{{\mathcal{R}}}
\def\gS{{\mathcal{S}}}

\def\gX{{\mathcal{X}}}

\def\sN{{\mathbb{N}}}

\def\sR{{\mathbb{R}}}

\usepackage[utf8]{inputenc} 
\usepackage[T1]{fontenc}    
\usepackage{url}            
\usepackage{booktabs}       
\usepackage{amsfonts}       
\usepackage{nicefrac}       
\usepackage{microtype}      
\usepackage{mathrsfs}
\usepackage{multirow}
\usepackage{algorithm}
\usepackage{algorithmic}
\usepackage{adjustbox}
\usepackage{bbm}
\usepackage{bm}
\usepackage{color}
\usepackage{mathrsfs}
\usepackage{amsmath} 
\usepackage[numbers]{natbib}
\usepackage{float}
\usepackage{subcaption}
\usepackage{amssymb}
\usepackage[dvipsnames]{xcolor}
\usepackage{enumitem}
\newtheorem{thm}{Theorem}[section]
\newtheorem{defn}[thm]{Definition}
\newtheorem{lemma}[thm]{Lemma}
\newtheorem{assume}[thm]{Assumption}
\newtheorem{remark}[thm]{Remark}

\newtheorem{prop}[thm]{Proposition}
\newenvironment{proof}{Proof:}{\hfill$\square$}
\usepackage[toc,page]{appendix}
\def\diag{\mathrm{diag}}
\newcommand{\norm}[1]{\left\|#1\right\|}

\begin{document}
\title{Global Convergence Analysis of Deep Linear Networks \\ with A One-neuron Layer} 
\author{Kun Chen\thanks{Equal Contribution.}
        \thanks{School of Mathematical Sciences, Peking University. \texttt{kchen0415@pku.edu.cn}}
        \and 
        Dachao Lin\footnotemark[1]
        \thanks{Academy for Advanced Interdisciplinary Studies, Peking University. \texttt{lindachao@pku.edu.cn}}
        \and
        Zhihua Zhang\thanks{School of Mathematical Sciences, Peking University. \texttt{zhzhang@math.pku.edu.cn}}
}
\maketitle

\begin{abstract}
	In this paper, we follow \citet{eftekhari2020training}'s work to give a non-local convergence analysis of deep linear networks.
	Specifically, we consider optimizing deep linear networks which have a layer with one neuron under quadratic loss.
	We describe the convergent point of trajectories with arbitrary starting point under gradient flow, including the paths which converge to one of the saddle points or the original point.
	We also show specific convergence rates of trajectories that converge to the global minimizer by stages.
	To achieve these results, this paper mainly extends the machinery in \cite{eftekhari2020training} to provably
	identify the rank-stable set and the global minimizer convergent set.
	We also give specific examples to show the necessity of our definitions. 
	Crucially, as far as we know, our results appear to be the first to give a non-local global analysis of linear neural networks from arbitrary initialized points, rather than the lazy training regime which has dominated the literature of neural networks, and restricted benign initialization in \cite{eftekhari2020training}. 
	We also note that extending our results to general linear networks without one hidden neuron assumption remains a challenging open problem.
\end{abstract}

\section{Introduction}	
Deep neural networks have been successfully trained with simple gradient-based methods, which require optimizing highly non-convex functions.
Many properties of the learning dynamic for deep neural networks are also present in the idealized and simplified case of deep linear networks. 
It is widely believed that deep linear networks could capture some important aspects of optimization in deep learning (\cite{saxe2013exact}). 
Therefore, many works \cite{hardt2016identity, arora2018optimization, arora2018convergence, bartlett2018gradient, shamir2019exponential, du2019width, hu2019provable, zou2019global, bah2021, eftekhari2020training} have tried to study this problem in recent years.
However, previous understanding mainly adopts local analysis or lazy training \cite{chizat2019lazy}, and there are few findings of the non-local analysis, even for linear networks.

\textbf{Local analysis of deep linear networks with quadratic loss.}
Several works analyzed linear networks with the quadratic loss.
\citet{bartlett2018gradient} provided a linear convergence rate of gradient descent with identity initialization by assuming that the target matrix is either close to identity or positive definite. 
\citet{bartlett2018gradient} also showed the necessity of the positive definite target under identity initialization (see \citet[Theorem 4]{bartlett2018gradient}).
\citet{arora2018convergence} also proved linear convergence of deep linear networks, by assuming that the initialization has a positive deficiency margin and is nearly balanced.
Later, a few works followed a similar idea with the neural tangent kernel (NTK) \cite{jacot2018neural} or lazy training \cite{chizat2019lazy} to establish convergence analysis.
\citet{du2019width} demonstrated that, if the width of hidden layers is all larger than the depth, gradient descent with Gaussian random initialization could converge to a global minimum at a linear rate. 
\citet{hu2019provable} improved the lower bound of width to be independent of depth, by utilizing orthogonal weight initialization, but requiring each layer to have the same width.
Moreover, \citet{wu2019global, zou2019global} obtained a linear rate convergence for linear ResNet \cite{he2016deep} with zero(-asymmetric) initialization, i.e., deep linear network with identity initialization.
Specifically, \citet{wu2019global} adopted zero-asymmetric initialization requiring a zero-initialized output layer and identity initialization for other layers.
Such unbalanced weight matrices lead to a small variation of output layer weight compared to the other layers, which is similar to local analysis.  
And \citet{zou2019global} applied identity initialization (for deep linear networks), but still required a lower bound for the width.
All the above works are not non-local analyses.

\textbf{Non-local analysis of deep linear networks with quadratic loss.}
The non-local analysis requires a more comprehensive analysis. 
As far as we know, current works mainly focused on gradient flow, i.e., gradient descent with an infinitesimal small learning rate. 
From the manifold viewpoint, \citet{bah2021} showed
that the gradient flow always converges to a critical point of the underlying functional. 
Moreover, they established that, for almost all initialization, the flow converges to a global minimum on the manifold of rank-$k$ matrices, where $k$ can be smaller than the largest possible rank of the induced weight. Hence, their work only ensured the convergence to minimizers in a constrained subset, which is not necessarily the global minimizer.
Additionally, they also provided a concrete example to display the existence of such rank unstable trajectories (see \citet[Remark 42]{bah2021}).
Following \citet{bah2021}, \citet{eftekhari2020training} provided a non-local convergence analysis for deep linear nets with quadratic loss. 
By assuming that one layer has only one neuron (including scalar output case) and the initialization is balanced, \citet{eftekhari2020training} elaborated that gradient flow converges to global minimizers starting from a restricted set. Moreover, \citet{eftekhari2020training} also confirmed that gradient flow could efficiently solve the problem by showing concrete linear convergence rates in the restricted set he defined.

In this work, we are interested in the \textbf{non-local analysis} of deep linear networks with the quadratic loss for \textbf{arbitrary initialization}.
To our knowledge, there was no non-local convergence analysis of gradient flow for deep linear nets in such a scheme. 

\subsection{Our Contributions}
In this paper, we analyze gradient flow for deep linear networks with quadratic loss following the setting of \cite{eftekhari2020training}. 
The main contributions of this paper are summarized as follows:
\begin{itemize}
	\item \textbf{Convergent result.}
	We first analyze the convergent behavior of trajectories.
	Compared to \citet{eftekhari2020training}, we define a more general rank-stable set of initialization to give \textit{almost surely} convergence guarantee to the global minimizer (Theorem \ref{thm:rank}).
	Moreover, we also describe a more general global minimizer convergent set to guarantee convergence to the global minimizer (Theorem \ref{thm:global}).
	
	Furthermore, inherited from the above results, we introduce the indicator of \textit{arbitrary} beginning point to decide the convergent point of the trajectory (Theorem \ref{thm:ab}).
	Our analysis covers the trajectories that converge to saddle points. 
	We emphasize that our analysis is beyond the lazy training scheme, and does not require the constrained initialization region mentioned in \cite{eftekhari2020training}.
	
	\item \textbf{Convergence rate.} 
	We also establish explicit convergence rates of the trajectories converging to the global minimizer. 
	Our convergence rates build on the fact that the dynamic of the singular value $s(t)$ can be divided into two stages: $s(t)$ decreases in the first stage and increases in the second stage.
	In the case where our convergent results declare that $\mW(t) \to s_1\vu_1\vv_1^T$, we show that in the worse case, there are three stages of convergence\footnote{We denote $s_1\vu_1\vv_1^T$ as the SVD of the best rank-one approximation of target, $\mW(t) = s(t)\vu(t)\vv(t)^\top$ as the SVD of the induced weight matrix at time $t$, $a_1(t) = \vu_1^\top\vu(t), b_1(t) = \vv_1^\top \vv(t)$, and $N$ is the number of layers, see Section \ref{sec:rates} for details.}:
	\begin{equation*}
	\begin{aligned}
	    \text{Stage 1: } &  1-a_1(t)b_1(t) = \gO\left([(N-2)t]^{-\frac{c_1}{N-2}}\right), \ s(t)=\Omega\left([(N-2)t]^{-\frac{N}{N-2}}\right) \text{ until } a_1(t)b_1(t) \geq 0; \\
	    \text{Stage 2: } &  1-a_1(t)b_1(t) = \gO\left([(N-2)t]^{-\frac{c_2}{N-2}}\right), \ s(t)=\Omega\left([(N-2)t]^{-\frac{N}{N-2}}\right) \text{ until } \dot{s}(t) \geq 0; \\
	    \text{Stage 3: } &  1-a_1(t)b_1(t) = \gO\left(e^{-c_5t}\right), \quad  s_1-s(t) = \gO \left(e^{-\min\{c_5,c_6\}t}\right),
	\end{aligned}
	\end{equation*}
	where $c_1, c_2$ are positive constants related to previous stages, target, initialization, and $c_5, c_6$ are analogous positive constants additionally related to depth $N$.
	We conclude that the rates begin from polynomial to linear convergence, and are heavily dependent on the initial magnitude of $a_1(0)+b_1(0)$.
	And our analysis is more comprehensive since \citet{eftekhari2020training} only gave linear convergence rates for the last stage.

	\item We conduct numerical experiments to verify our findings. 
	Though gradient descent seldom converges to strict saddle points \cite{lee2016gradient}, we discover that our analysis of gradient flow reveals the long stuck period of trajectory under gradient descent in practice and the transition of the convergence rates for trajectories.
\end{itemize}

\subsection{Additional Related work}

\textbf{Exponentially-tailed loss.} 
There is much literature \cite{gunasekar2018implicit, nacson2019convergence, ji2020directional, ji2019gradient, lyu2019gradient} focusing on classification tasks under exponentially-tailed loss, such as logit loss or cross-entropy loss. 
Though this paper mainly focuses on the quadratic loss, we also list some findings related to linear networks. 
Specifically, \citet{gunasekar2018implicit, nacson2019convergence} proved the convergence to a max-margin solution assuming the loss converged to global optima. 
\citet{lyu2019gradient, ji2019gradient, ji2020directional} also demonstrated the convergence to the max-margin solution under weaker assumptions that the initialization has zero classification error. 
These analyses focused on the final phase of training, which is still not a global analysis.
\citet{lin2021faster} showed a global analysis for directional convergence of deep linear networks. Their results also covered arbitrary initialization, but they required the spherically symmetric data assumption.

\textbf{Global landscape analysis.}
Except for the non-local trajectory analysis, there is another line of works on non-local landscape analysis (see, e.g.,
\cite{kawaguchi2016deep,lu2017depth,laurent2018deep,nouiehed2018learning,zhang2019depth, nguyen2018loss,li2018over,ding2019spurious,liang2018understanding,liang2018adding,liang2019revisiting,venturi2018spurious,safran2017spurious,achour2021global} and the surveys \cite{sun2020global,sun2020optimization}) which analyze the properties of stationary points, local minima, strict saddle points, etc.
These works draw a whole picture of the benign landscape of deep linear networks, which provides a potential guarantee of the trajectory analysis, and motivates our work.

\subsection{Organization}
The remainder of the paper is organized as follows. 
We present some preliminaries in Section \ref{sec:pre}, including the notation, assumptions, and some preparation of our narration.
In Section \ref{sec:convergent}, we analyze the convergent points of trajectories progressively.
We show the rank-stable initialized set in Subsection \ref{subsec:rank}, the global minimizer convergent set of initialization in Subsection \ref{subsec:global-con}, and the convergent behavior of arbitrary initialization in Subsection \ref{subsec:all}. 
We list some examples to support our results in Subsection \ref{subsec:examples}.
In Section \ref{sec:rates}, we give explicit convergence rates for the trajectories converging the global minimizer.
In Section \ref{sec:exp}, we perform numerical experiments to support our theoretical results.
Finally, we conclude our work in Section \ref{sec:con}.

\section{Preliminaries}\label{sec:pre}
In this paper, we consider the optimization of deep linear network under squared loss:
\[ \min_{\mW_1, \dots, \mW_N} L^N(\mW_1, \dots, \mW_N):= \norm{\mW_N\cdots \mW_1\mX-\mY}_F^2, \ N \geq 2, \]
where $\mW_i \in \sR^{d_i \times d_{i-1}}, d_0=d_x,d_N=d_y, \mX\in\sR^{d_x \times m}, \mY\in\sR^{d_y \times m}$.
For brevity, we denote the induced weight as $\mW = \mW_N\cdots \mW_1\in \sR^{d_y \times d_x}$, and $\mW_{\mN} = (\mW_1, \dots, \mW_N) \in \sR^{d_1 \times d_0} \times \dots \times \sR^{d_N \times d_{N-1}}$.

\paragraph{Notation.}
We denote vectors by lowercase bold letters (e.g., $ \bm{u}, \bm{x}$), and matrices by capital bold letters (e.g., $ \bm{W} = [w_{i j}] $). We set $[a, b] = \{a, \dots, b\}, [a] = [1, a], \forall a,b \in \sN$ if no ambiguity with the closed interval, and $s_i(\mZ)$ as the $i$-th largest singular of $\mZ$.
We adopt $\norm{\cdot}$ as the standard Euclidean norm ($\ell_2$-norm) for vectors, and $\norm{\cdot}_F$ as Frobenius norm for matrices. The convergence of vectors and matrices in this paper is defined under the standard Euclidean norm and Frobenius norm. 
We use the standard $\gO(\cdot), \Omega(\cdot)$ and $\Theta(\cdot)$ notation to hide universal constant factors.

We integrate our assumptions in this paper below:
\begin{assume}\label{ass}We assume that the target, data, network and the initialization satisfy:
    \quad 
    
	\begin{itemize}
		\item Target: $\mZ := \mY\mX^\top \in \sR^{d_y \times d_x}$ has different nonzero singular values, i.e., $s_1(\mZ) > \dots > s_d(\mZ)>0$, where $d = \mathrm{rank}(\mY\mX^\top)$. 
		\item Data: $\mX\mX^\top=\mI_{d_x}$. 
		\item Network: $r := \min_{j \leq N} d_j=1$.
		\item Initialization: $\mW_i\mW_i^\top=\mW_{i+1}^\top\mW_{i+1}, \forall i\in[N-1]$.
	\end{itemize}
\end{assume}

For the target, it is reasonable to assume different nonzero singular values in practice, since matrices with the same singular values have zero Lebesgue measure. 
The last three assumptions are the same as \citet{eftekhari2020training}. 
The second assumption shows that the data is statistically whitened, which is common in the analysis of linear networks (\cite{arora2018convergence,bartlett2018gradient}). The network assumption includes the scalar output case. 
As mentioned in \cite{eftekhari2020training}, this case is significant as it corresponds to the popular spiked covariance
model in statistics and signal processing (\cite{eftekhari2019moses, johnstone2001distribution, vershynin2012close, berthet2013optimal, deshpande2014information}), to name a few.
Moreover, the case $r=1$ appears to be the natural beginning building block for understanding the behavior of trajectory.
Finally, the last assumption is the common initialization technique for linear networks, which has appeared in \cite{hardt2016identity,bartlett2018gradient,arora2018convergence,arora2018optimization, arora2019implicit}.

From $\mX\mX^\top=\mI_{d_x}$ in Assumption \ref{ass}, we can simplify the problem as
\begin{equation}\label{eq:obj}
\min_{\mW_1, \dots, \mW_N} L^N(\mW_1, \dots, \mW_N) := \norm{\mW_N\cdots \mW_1-\mZ}_F^2+\norm{\mY}_F^2-\norm{\mZ}_F^2.
\end{equation} 
Hence, we call $\mZ$ as the target matrix. Moreover, we focus on the standard gradient flow method:
\begin{equation}\label{eq:flow-wi}
    \dot{\mW_i}(t) := \frac{d\mW_i(t)}{dt} = -\nabla_{\mW_i}L^N(\mW_{\mN}(t)), \forall j \in[N], t \geq 0.
\end{equation}

Under the balanced initialization in Assumption \ref{ass}, i.e., $\mW_i\mW_i^\top=\mW_{i+1}^\top\mW_{i+1}, \forall i\in[N-1]$, we have the induced weight flow of $\mW(t) := \mW_N(t) \cdots \mW_1(t)$ following \citet[Theorem 1]{arora2018optimization}:
\begin{equation}\label{eq:induced}
    \begin{aligned}
    \dot{\mW}(t) &= -\gA(\mW(t)) := -\sum_{j=1}^{N} \left(\mW(t)\mW(t)^\top \right)^{\frac{N-j}{N}}\nabla L^1(\mW(t))\left(\mW(t)^\top\mW(t)\right)^{\frac{j-1}{N}} \\
    &= -\sum_{j=1}^{N} \left(\mW(t)\mW(t)^\top \right)^{\frac{N-j}{N}}(\mW(t)-\mZ) \left(\mW(t)^\top\mW(t)\right)^{\frac{j-1}{N}}.
    \end{aligned}
\end{equation}

It is known that the induced flow in Eq.~\eqref{eq:induced} admits an analytic singular value decomposition (SVD), see Lemma 1
and Theorem 3 in \citet{arora2019implicit} for example.
Since $\mathrm{rank}(\mW) \leq 1$ from Assumptions \ref{ass}, we can denote the SVD of $\mW(t)$ as $\mW(t) = s(t)\vu(t)\vv(t)^\top$ if $\mW(t) \neq \bm{0}$. Here, $s(t),\vu(t),\vv(t)$ are all analytic functions of $t$. Moreover, 
$s(t) \in \sR, \vu(t) \in \sR^{d_y}, \vv(t) \in \sR^{d_x}$, and $\norm{\vu(t)}=\norm{\vv(t)}=1$. 
Previous work has already shown the variation of these terms:
\begin{align}
\dot{\vu}(t) &= s(t)^{1-\frac{2}{N}}\left(\mI_{d_y}-\vu(t)\vu(t)^\top\right)\mZ\vv(t), \label{eq:grad-u} \\
\dot{\vv}(t) &= s(t)^{1-\frac{2}{N}}\left(\mI_{d_x}-\vv(t)\vv(t)^\top\right)\mZ^\top\vu(t), \label{eq:grad-v}\\
\dot{s}(t) &= N s(t)^{2-\frac{2}{N}}\left(\vu(t)^\top\mZ\vv(t)-s(t)\right). \label{eq:grad-s}
\end{align}
Readers can discover the derivation of $\dot{s}(t)$ in \citet[Theorem 3]{arora2019implicit}, and $\dot{\vu}(t), \dot{\vv}(t)$ in \citet[Eq.~(139)]{eftekhari2020training} or \citet[Lemma 2]{arora2019implicit} with some simplification. We also give the derivation of $\dot{\vu}(t), \dot{\vv}(t)$ in Lemma \ref{lemma:dot-s-u-v}.


To describe the solution obtained by flow, we also need the full SVD of target as $\mZ = \mU\mD\mV^\top = \sum_{i=1}^{d}s_i\vu_i\vv_i^\top$ with $s_1 > \dots > s_d > 0$, orthogonal matrices $\mU = [\vu_i] \in \sR^{d_y \times d_y}$ and $\mV = [\vv_i] \in \sR^{d_x \times d_x}$, and $\mD = \left( \begin{smallmatrix} \diag\{\vs\} & \bm{0} \\
\bm{0} & \bm{0} \end{smallmatrix} \right) \in \sR^{d_y \times d_x}$ where $\vs = (s_1,\dots,s_d)^\top \in \sR^d$. 
And the best rank-one approximation matrix $\mZ_1 = s_1\vu_1\vv_1^\top$.
Note that $\mZ_1$ is the unique solution of problem Eq.~\eqref{eq:obj}, because
$\mZ$ has a nontrivial spectral gap by Assumption \ref{ass} (see \cite[Section 1]{golub1987generalization}).
For brevity, we define $s_k = 0, \forall k > d$.
We adopt the projection length of $\vu(t), \vv(t)$ to each $\vu_i$ and $\vv_i$ as $a_i(t) = \vu_i^\top\vu(t), \forall i\in [d_y]$ and $b_j(t) = \vv_j^\top\vv(t), \forall j\in [d_x]$. So we have
\begin{equation}\label{eq:uv-exp}
    \vu(t)=\sum_{i=1}^{d_y}a_i(t)\vu_i, \vv(t)=\sum_{j=1}^{d_x}b_j(t)\vv_j, \ \mathrm{and} \ \sum_{i=1}^{d_y} a^2_i(t)=\sum_{j=1}^{d_x} b^2_j(t)=1.
\end{equation}
Then we get
\begin{equation}\label{eq:uzv}
\begin{aligned}
&\vu(t)^\top\mZ\vv(t) = \sum_{i=1}^{d_y} \sum_{j=1}^{d_x} a_i(t)b_j(t) \vu_i^\top \mZ \vv_j = \sum_{i=j=1}^{d} a_i(t)b_j(t) \vu_i^\top \mZ \vv_j = \sum_{j=1}^{d} s_j a_j(t)b_j(t), \\
&\mZ\vv(t) = \sum_{j=1}^{d_x} b_j(t) \mZ \vv_j =\sum_{j=1}^{d} s_j b_j(t) \vu_j, \quad \ \; \mZ^\top\vu(t) = \sum_{i=1}^{d_y} a_i(t) \mZ^\top \vu_i = \sum_{i=1}^{d} s_i a_i(t) \vv_i.
\end{aligned}
\end{equation}
where we uses the fact that $\mZ^\top\vu_i = \bm{0}, \forall i > d$, $\mZ\vv_j = \bm{0}, \forall j > d$, and $\vu_i^\top \mZ \vv_j = 0, \vu_i^\top \mZ \vv_i = s_i, \forall i\neq j\leq d$.
Hence, we have the gradient flow of each item:
\begin{equation}\label{eq:grad-ab}
    \begin{aligned}
    \dot{a}_i(t) & \stackrel{\eqref{eq:grad-u}}{=} s(t)^{1-\frac{2}{N}}\vu_i^\top\left(\mI_{d_y}-\vu(t)\vu(t)^\top\right)\mZ\vv(t) \stackrel{\eqref{eq:uzv}}{=} s(t)^{1-\frac{2}{N}} \bigg(s_i b_i(t)-a_i(t)\sum_{j=1}^d \left[s_j a_j(t)b_j(t)\right]\bigg), \forall i \in [d_y]. \\
    \dot{b}_i(t) & \stackrel{\eqref{eq:grad-v}}{=} s(t)^{1-\frac{2}{N}}\vv_i^\top\left(\mI_{d_x}-\vv(t)\vv(t)^\top\right)\mZ^\top\vu(t) \stackrel{\eqref{eq:uzv}}{=} s(t)^{1-\frac{2}{N}}\bigg(s_i a_i(t)-b_i(t)\sum_{j=1}^d \left[s_j a_j(t)b_j(t)\right]\bigg), \forall i \in [d_x].
    \end{aligned}
\end{equation}

Before we provide our results, we first state several useful invariance during the whole training dynamic as follows, which is crucial to our proofs.
\begin{prop}\label{prop:base}
	If not mentioned specifically, we assume $s(0)>0$. We have the following useful properties: 
	\begin{itemize}
		\item 1). If $s(0) > 0$, then $\forall t\geq 0, s(t) > 0$. Otherwise, $s(0)=0$, then $\forall t\geq 0, s(t)=0$ (i.e., $\mW(t)=\bm{0}$).
		\item 2). $\vu(t)^\top\mZ\vv(t)$ is non-decreasing and converges.
		\item 3). $\vu(t)^\top\mZ_1\vv(t)$ is non-decreasing and converges.
		\item 4). For all $t \geq 0$, $a_{i}(t)+b_{i}(t)$ has the same sign with $a_i(0)+b_{i}(0)$, i.e., $a_i(t)+b_i(t)$ is identically zero if $a_i(0)+b_i(0) = 0$, is positive if $a_i(0)+b_i(0)>0$, and is negative if  $a_i(0)+b_i(0)<0$.
		\item 5). If for some $k \in [0, d-1]$, $a_i(0)+b_i(0) = 0, \forall i \in [k]$ (if $k=0$, then no such assumptions) and $a_{k+1}(0)+b_{k+1}(0) \neq 0$, then $|a_{k+1}(t)+b_{k+1}(t)|$ is non-decreasing, and $\lim_{t \to +\infty}a_{k+1}(t)+b_{k+1}(t)$ exists.
	\end{itemize}
\end{prop}

\section{Convergent Behavior of Trajectories}\label{sec:convergent}
\subsection{Rank-stable Set}\label{subsec:rank}
By \citet[Theorem 5]{bah2021} (Theorem \ref{thm:aux-3}), $(\mW_1(t), \dots, \mW_N(t))$ always converges to a critical point of $L^N$ as $t \to +\infty$. Hence, we can define $\overline{\mW}:=\lim_{t \to +\infty} \mW(t)$, and $\bar{s} = \lim_{t \to +\infty}s(t) = \lim_{t \to +\infty} \norm{\mW(t)}_F$. 
To specify the convergent point, we define rank-$ r $ set following \cite{eftekhari2020training} as
\[ \gM_{r} = \{\mW: \mathrm{rank}(\mW)=r \}. \]

Furthermore, as mentioned in \citet[Lemma 3.3]{eftekhari2020training} (Lemma \ref{aux-lemma:1}), we have $\mathrm{rank}(\mW(t)) = \mathrm{rank}(\mW(0))=1, \forall t \geq 0$ if $\mW(0) \neq \bm{0}$.
However, the limit point $\overline{\mW}(t)$ might not belong to $\gM_{1}$ because $\gM_{1}$ is not closed, see \cite[Lemma 3.4]{eftekhari2020training}.
To exclude the zero matrix ($\bar{s}=0$) as the limit point of gradient flow, \citet{eftekhari2020training} introduced a restricted initialization set:
\[ \gN_{\alpha}(\mZ) = \{\mW: \mW \stackrel{\mathrm{SVD}}{=} \vu \cdot s \cdot \vv^\top, \ s > \alpha s_1-s_2 > 0, \vu^\top \mZ_1 \vv > \alpha s_1 \}, \alpha \in (s_2/s_1, 1]. \]
While we find another rank-stable set $\gR_b(\mZ)$ below with similar rank-stable property shown in Lemma \ref{lemma:hb}.
\[ \gR_b(\mZ) = \{\mW: \mW \stackrel{\mathrm{SVD}}{=} \vu \cdot s \cdot \vv^\top, \ s>b, \vu^\top \mZ \vv>b \}, b>0. \]

\begin{lemma}[Extension of Lemma 3.7 in \citet{eftekhari2020training}]\label{lemma:hb}
	Under Assumption \ref{ass}, for gradient flow initialized at $\mW(0)) \in \gR_b(\mZ)$, the limit point exists and satisfies $\overline{\mW}=\lim_{t \to +\infty} \mW(t) \in \gM_{1}$.
\end{lemma}
\begin{proof}
    We only need to prove that $\mW(t) \in \gR_b(\mZ), \forall t\geq 0$.
	Since $\mW(0) \in \gR_b(\mZ)$, then $\vu(0)^\top \mZ \vv(0) > b$. Thus, we have $\vu(t)^\top \mZ \vv(t) > b, \forall t\geq 0$ by 2) in Proposition \ref{prop:base}. 
	Hence, $\mW(t)$ only could leave the region $\gR_b(\mZ)$ when $s(T) = b$ for some time $T>0$.
    
    Since $\mW(0) \in \gR_b(\mZ)$, then $s(0) > b>0$. Thus, we have $s(T) > 0$ by 1) in Proposition \ref{prop:base}. 
	Therefore, we get
	\begin{equation*}
	\dot{s}(T) \stackrel{\eqref{eq:grad-s}}{=} Ns(T)^{2-\frac{2}{N}}\left(\vu(T)^\top\mZ\vv(T)-s(T)\right) = N s(T)^{2-\frac{2}{N}}\left(\vu(T)^\top\mZ\vv(T)-b\right) > 0,
	\end{equation*}
	which pushes the singular value up and thus pushes the induced flow back into $\gR_b(\mZ)$. Contradiction!
\end{proof}

\begin{remark}
    Indeed, our rank-stable set $\gR_b(\mZ)$ is more general than $\gN_{\alpha}(\mZ)$. 
    For any $\alpha \in (s_2/s_1, 1], \mW \in \gN_{\alpha}(\mZ)$, since $|\vu^\top\left(\mZ-\mZ_1\right)\vv| \leq s_2$, we get $ \vu^\top \mZ \vv \geq  \vu^\top \mZ_1 \vv -s_2 \geq \alpha s_1 -s_2>0 $. 
    Hence, we can find some $b > 0$, such that $\mW \in \gR_{b}(\mZ)$. 
    Moreover, when $\mW = s_2\vu_2\vv_2^\top$, we can see $\mW \in \gR_{s_2/2}(\mZ)$, but $\mW \not\in \gN_{\alpha}(\mZ), \forall \alpha \in (s_2/s_1, 1]$, since $\vu^\top \mZ_1 \vv=0$.
    Additionally, we will see the necessity of our rank-stable set by showing counterexamples in Section \ref{subsec:examples}.
\end{remark}

Applying the same analysis as \citet[Theorem 3.8]{eftekhari2020training}, we could obtain almost surely convergence to the global minimizer from the initialization in our rank-stable set.

\begin{thm}[Extension of Theorem 3.8 in \citet{eftekhari2020training}]\label{thm:rank}
	Under Assumption \ref{ass}, gradient flow converges to a global minimizer of the original problem Eq.~\eqref{eq:obj} from the initialization in $\mW(0) \in \gR_b(\mZ), b>0$, outside of a subset with Lebesgue measure zero.
\end{thm}

\begin{proof}
    From Lemma \ref{lemma:hb}, we already know the convergent point $\overline{\mW}(t)$ is still rank-one. 
    Hence, using the facts shown in \citet[Theorem 28]{bah2021} (Theorem \ref{thm:aux-3}) that gradient flows avoid strict saddle points almost surely, and \citet[Proposition 33]{bah2021} (Proposition \ref{thm:aux-4}) that $L^1(\mW)$ on $\gM_1$ satisfies the strict saddle point property, we could see gradient flow of $\mW(t)$ converges to a global minimizer almost surely.
\end{proof}

\subsection{Global Minimizer Convergent Set}\label{subsec:global-con}
Section \ref{subsec:rank} mainly analyzes the behavior of $s(t)$ to guarantee the rank of limit point is not degenerated.
Though Theorem \ref{thm:rank} ensures the almost surely convergence to the global minimizer, there still have some bad trajectories which converge to saddle points, such as $s_i\vu_i\vv_i^\top, i\neq 1$.
In this subsection, we move on to give another restricted initialization set to guarantee the global minimizer convergence without excluding a zero measure set.
Our strategy mainly adopts singular vector analysis. 
To analyze the behavior of $\vu(t), \vv(t)$, we need the following lemmas. 
In the following, we always assume $\mW(0) \neq \bm{0}$ to ensure well-defined $\vu(t), \vv(t)$.

\begin{lemma}\label{lemma:baruv}
	There exists a sequence $\{t_n\}$ with $t_n \to +\infty$, such that $\lim_{n \to +\infty}\left(\mI_{d_y}-\vu(t_n)\vu(t_n)^\top\right)\mZ\vv(t_n) = \bm{0}$, and $\lim_{n \to +\infty}\left(\mI_{d_x}-\vv(t_n)\vv(t_n)^\top\right)\mZ^\top\vu(t_n) = \bm{0}$.
	More specifically, we have
	\begin{align}
	\lim_{n \to +\infty} \bigg(\sum_{j=1}^d s_j a_j(t_n)b_j(t_n)\bigg) a_i(t_n) - s_i b_i(t_n) = 0, \forall i \in [d_y], \label{eq:abti} \\
	\lim_{n \to +\infty} \bigg(\sum_{j=1}^d s_j a_j(t_n)b_j(t_n)\bigg) b_i(t_n) - s_i a_i(t_n) = 0, \forall i \in [d_x]. \label{eq:bati} 
	\end{align}
	
	Furthermore, if there exists $i_0 \in [d]$, such that $\lim_{n \to +\infty} a_{i_0}(t_n)+b_{i_0}(t_n)$ exists and is not zero, we could obtain
	\begin{equation}\label{eq:uzv-s}
		\lim_{n \to +\infty} \vu(t_n)^\top\mZ\vv(t_n) = \lim_{n \to +\infty} \sum_{j=1}^d s_j a_j(t_n)b_j(t_n) = s_{i_0}.
	\end{equation}
\end{lemma}

\begin{lemma}\label{lemma:getab}
	Suppose there exists a sequence $\{t_n\}$ that $t_n \to +\infty$, and for some $k<d, a_i(t_n)+b_i(t_n)=0, \forall i \in [k], n \geq 0$. Then if $\lim_{n \to +\infty}\sum_{j=1}^d s_j a_j(t_n)b_j(t_n) = s_{k+1}$, we could obtain 
	\begin{equation}\label{eq:ab-limit}
	    \lim_{n \to +\infty} a_i(t_n) = \lim_{n \to +\infty} b_j(t_n) = 0, \forall i, j \neq k+1, i \in[d_y], j \in[d_x]; \ \lim_{n \to +\infty} a_{k+1}(t_n)b_{k+1}(t_n) = 1.
	\end{equation}
	That is, $\lim_{n \to +\infty} \vu(t_n)\vv(t_n)^\top = \vu_{k+1}\vv_{k+1}^\top$.
\end{lemma}

Now we define the global minimizer convergent set as 
\[ \gG_b(\mZ) = \{\mW: \mW \stackrel{\mathrm{SVD}}{=} \vu \cdot s \cdot \vv^\top, \ s>b, \vu^\top \mZ_1 \vv>b \}. \]
The only difference compared to $\gR_b(\mZ)$ is that we replace $\mZ$ to $\mZ_1$.
\begin{thm}[Extension of Theorem 3.8 in \citet{eftekhari2020training}]\label{thm:global}
	Under Assumption \ref{ass}, gradient flow converges to a global minimizer of the original problem Eq.~\eqref{eq:obj} from the initialization $\mW(0) \in \gG_b(\mZ), b>0$.
\end{thm}

\begin{proof}
	Our proof is separated to the following steps.
	\begin{itemize}
		\item \textbf{Step 1.} 
		Since $\mW(0) \in \gG_b(\mZ)$, we have $s_1 a_1(0) b_1(0) = \vu(0)^\top \left(s_1\vu_1 \vv_1^\top\right)\vv(0) =\vu(0)^\top\mZ_1 \vv(0)> b$. By 3) in Proposition \ref{prop:base}, we obtain $\vu(t)^\top\mZ_1\vv(t) = s_1\vu(t)^\top\vu_1\vv_1^\top\vv(t) = s_1a_1(t)b_1(t)$ is non-decreasing. Hence, $s_1a_1(t)b_1(t)\geq s_1 a_1(0) b_1(0)>b>0$, leading to $|a_1(t)+b_1(t)|>0$.
		Additionally, by 5) in Proposition \ref{prop:base}, we get $|\lim_{t \to +\infty} a_1(t)+b_1(t)| \geq |a_1(0)+b_1(0)| > 0$.
		Thus, applying Lemma \ref{lemma:baruv} with $i_0=1$, Eq.~\eqref{eq:uzv-s} holds, i.e., there exists a sequence $\{t_n\}$ with $t_n \to +\infty$, s.t., $\lim_{n \to +\infty} \vu(t_n)^\top\mZ\vv(t_n) = \lim_{n \to +\infty} \sum_{i=1}^d s_j a_j(t_n)b_j(t_n) = s_{1}$. 
		
		\item \textbf{Step 2.} From \textbf{Step 1}, $\lim_{n \to +\infty} \sum_{i=1}^d s_j a_j(t_n)b_j(t_n) = s_{1}$. Then we can employ Lemma \ref{lemma:getab} with $k=0$, showing that $\vu(t_n)\vv(t_n)^\top \to \vu_1\vv_1^\top$.
		
		\item \textbf{Step 3.} From \textbf{Step 1}, we know that $\vu(t_n)^\top\mZ \vv(t_n) \to s_1 > 0$, leading to $\exists N > 0$, $\vu(t_N)^\top \mZ \vv(t_N)>0$. 
		Moreover, since $\mW(0) \in \gG_b(\mZ)$, we have $s(0)>0$. Thus, we have $s(t_N)>0$ by 1) in Proposition \ref{prop:base}. Hence, $\mW(t_N) \in \gR_b(\mZ)$ for some $b>0$. Thus, by Lemma \ref{lemma:hb}, we have $\bar{s}>0$.
		
		\item \textbf{Step 4.} Finally, taking $t_n \to +\infty$ in Eq.~\eqref{eq:grad-s}, we obtain $0 = N\bar{s}^{2-\frac{2}{N}}\left(s_1-\bar{s}\right)$. While by \textbf{Step 3}, $\bar{s}>0$. Hence, $\bar{s}=s_1$. 
	\end{itemize}
	Combining \textbf{Step 2} and \textbf{Step 4}, we obtain $\mW(t_n)$ converges to the global minimizer $s_1\vu_1\vv_1^\top$.
	Finally, we note that by Theorem 5 in \citet{bah2021}, $\mW(t)$ converges. Hence, $\mW(t) \to s_1\vu_1\vv_1^\top$, which is the global minimizer.
\end{proof}

\begin{remark}
    We underline that Theorem \ref{thm:global} does not need to leave out a zero measure set comparing to Theorem \ref{thm:rank}.
    Meanwhile, $\gG_b(\mZ)$ is more general than $\gN_{\alpha}(\mZ)$ as well, since we have less constraint for $\vu^\top\mZ_1\vv$ in $\gG_b(\mZ)$.
    Moreover, we will see the necessity of our global minimizer convergent set by showing counterexamples in Section \ref{subsec:examples}.
\end{remark}


\subsection{Convergence Analysis for All Initialization}\label{subsec:all}
Now we change our perspective to the whole initialization, which includes the trajectories that converge to saddle points, instead of the global minimizer. 
The main conclusion is that the convergent point is decided by the indicator of initialization: $a_i(0)+b_i(0), i \in [d]$.

\begin{thm}\label{thm:ab}
	Under Assumption \ref{ass}, and assume $s(0)>0$. (I) If $k \in [0, d-1], a_i(0)+b_i(0) = 0, \forall i \in [k]$ (if $k= 0$, then no such assumptions), and $a_{k+1}(0)+b_{k+1}(0)\neq 0$, we have $\mW(t) \to s_{k+1}\vu_{k+1}\vv_{k+1}^\top$. (II) Otherwise, i.e., $a_i(0)+b_i(0) = 0, \forall i \in [d]$, then we have $\mW(t) \to \bm{0}$.
\end{thm}
\begin{proof}
	(I) For the first conclusion, our proof is separated to the following steps.
	\begin{itemize}
		\item \textbf{Step 1.} From 5) in Proposition \ref{prop:base}, we could see $|\lim_{t \to +\infty} a_{k+1}(t)+b_{k+1}(t)| \geq |a_{k+1}(0)+b_{k+1}(0)| > 0, \forall t\geq 0$. Thus, applying Lemma \ref{lemma:baruv}, we obtain $\exists \{t_n\}$ with $t_n \to +\infty$, s.t.,
		\[ \lim_{n \to +\infty} \vu(t_n)^\top\mZ\vv(t_n) = \lim_{n \to +\infty} \sum_{i=1}^d s_j a_j(t_n)b_j(t_n) = s_{k+1}. \]
		
		\item \textbf{Step 2.} Hence, by \textbf{Step 1} and Lemma \ref{lemma:getab}, we could obtain $\vu(t_n)\vv(t_n)^\top \to \vu_{k+1}\vv_{k+1}^\top$.
		
		\item \textbf{Step 3.} By \textbf{Step 1}, $\vu(t_n)^\top\mZ\vv(t_n) \to s_{k+1} > 0$, leading to $\exists N > 0$, $\vu(t_N)^\top \mZ \vv(t_N)>0$. 
		Moreover, since $s(0)>0$, we have $s(t_N)>0$ by 1) in Proposition \ref{prop:base}. Hence, $\mW(t_N) \in \gR_b(\mZ)$ for some $b>0$. Thus, by Lemma \ref{lemma:hb}, we have $\bar{s}>0$.
		
		\item \textbf{Step 4.} Finally, taking $t_n \to +\infty$ in Eq.~\eqref{eq:grad-s}, by \textbf{Step 1}, we obtain $0 = N\bar{s}^{2-\frac{2}{N}}\left(s_{k+1}-\bar{s}\right)$. While we have $\bar{s}>0$ by \textbf{Step 3}. Hence, $\bar{s}=s_{k+1}$. 
	\end{itemize}
	Combining \textbf{Step 2} and \textbf{Step 4}, we obtain that $\mW(t_n)$ converges to $s_{k+1}\vu_{k+1}\vv_{k+1}^\top$.
	Finally, we note that by Theorem 5 in \citet{bah2021}, $\mW(t)$ converges. Hence, $\mW(t) \to s_{k+1}\vu_{k+1}\vv_{k+1}^\top$.

	(II) For the second conclusion, i.e, assume $a_i(0)+b_i(0) = 0, \forall i \in [d]$. Then by 4) in Proposition \ref{prop:base}, we obtain $a_i(t)+b_i(t) = 0, \forall i \in[d], t\geq 0$. Hence, we get 
	\[ \vu(t)^\top\mZ \vv(t) \stackrel{\eqref{eq:uzv}}{=}\sum_{j=1}^{d} s_j a_j(t)b_j(t) = -\sum_{j=1}^{d} s_j a^2_j(t) \leq 0 \Rightarrow \dot{s}(t) \stackrel{\eqref{eq:grad-s}}{=} N s(t)^{2-\frac{2}{N}}\left(\vu(t)^\top\mZ\vv(t)-s(t)\right) \leq -N s(t)^{3-\frac{2}{N}}. \]
	By solving the ordinary differential equation (ODE) above, we derive that
	\[ \frac{N}{2-2N}s(t)^{\frac{2}{N}-2} -\frac{N}{2-2N}s(0)^{\frac{2}{N}-2}\le -Nt \Rightarrow s(t)^{\frac{2}{N}-2} - s(0)^{\frac{2}{N}-2}\geq (2N-2)t. \]
	Therefore, we obtain
	\[ s(t)^{2-\frac{2}{N}} \leq \left[s(0)^{\frac{2}{N}-2} + (2N-2)t \right]^{-1} \Rightarrow \bar{s} = \lim_{t \to +\infty}s(t) \leq \lim_{t \to +\infty}\left[s(0)^{\frac{2}{N}-2} + (2N-2)t \right]^{-\frac{N}{2N-2}} = 0. \]
	That is, $\mW(t) \to \bm{0}$.
\end{proof}

\begin{remark}
	We note that $s(0)=0$ indicates that $\mW(0)=0$, and $\mW(t)=\bm{0}, \forall t\geq 0$, which is a trivial case.
	Moreover, we also give a convergence rate for the second conclusion, i.e., the rate for the convergence to the original point.
\end{remark}

\subsection{Some Intuitive Examples}\label{subsec:examples}

Previous sections have shown the convergent behavior of arbitrary initialization. To give a better understanding of our results and the training behavior, we list some examples below. 

\paragraph{Example 1.}
If $\mW(0) = -s(0)\vu_i\vv_i^\top $ for some $i \in [\min\{d_y, d_x\}]$, we have $\forall t \geq 0, \dot{\vu}(t)=\bm{0}, \dot{\vv}(t)=\bm{0}$ and $\vu(t)^\top\mZ\vv(t) = s_i$ from Eqs.~\eqref{eq:grad-u} and \eqref{eq:grad-v}. Thus, we obtain the ODE of $s(t)$ as follows
\[ \dot{s}(t) \stackrel{\eqref{eq:grad-s}}{=} N s(t)^{2-\frac{2}{N}}\left(\vu(t)^\top\mZ\vv(t)-s(t)\right) = -N s(t)^{2-\frac{2}{N}}\left(s_i+s(t)\right) \leq 0. \]
So we could obtain $s(t) \to 0$. We see that $\mW(t) \to \bm{0}$, which is a rank-unstable trajectory. Thus, the gradient flow of Eq.~\eqref{eq:obj} does not converges to a global minimizer. 

\paragraph{Remark.} We note that our Theorem \ref{thm:ab} covers this example by choosing $k=d$, i.e., the second case in Theorem \ref{thm:ab}.
Moreover, we can see $-s(0)\vu_i\vv_i^\top \not\in \gR_{b}(\mZ)$, since $-\vu_i^\top\mZ\vv_i =-s_i \leq 0$. Thus, we could not further improve our definition of $\gR_{b}(\mZ)$. 

\paragraph{Example 2.}
If $\mW(0) = s(0)\vu_i\vv_i^\top$ for some $i \in [d]$ and $s(0)>0$, then from Eqs.~\eqref{eq:grad-u} and \eqref{eq:grad-v}, we obtain $\dot{\vu}(t)=\bm{0}, \dot{\vv}(t)=\bm{0}$ and $\vu(t)^\top\mZ\vv(t) = s_i$,  $\forall t\geq 0$.
Thus, we obtain the ODE of $s(t)$ as follows
\[ \dot{s}(t) \stackrel{\eqref{eq:grad-s}}{=} Ns(t)^{2-\frac{2}{N}}\left(\vu(t)^\top\mZ\vv(t)-s(t)\right) = N s(t)^{2-\frac{2}{N}}\left(s_i-s(t)\right). \]
Hence, we obtain $s(t) \to s_i$. 
Once $i\neq 1$, we could see that $\mW(t) \to s_i\vu_i\vv_i^\top$, i.e., the gradient flow of Eq.~\eqref{eq:obj} does not converges to a global minimizer. 

\paragraph{Remark.} We note that our Theorem \ref{thm:ab} covers this example by choosing $k=i$. Moreover, if $i \neq 1$, we can see $s(0)\vu_i\vv_i^\top \not\in \gG_{b}(\mZ)$, since $\vu_i^\top\mZ_1\vv_i=0$. Thus, we could not further improve our definition of $\gG_{b}(\mZ)$. 

\paragraph{Example 3.} \citet[Remark 42]{bah2021}: If $\mZ \succeq \bm{0}$, and $s(0)>0, \vu(0)=-\vv(0)$. Then from Eqs.~\eqref{eq:grad-u} and \eqref{eq:grad-v}, we obtain $\dot{\vu}(t) = -\dot{\vv}(t)$ if $\vu(t) = -\vv(t)$. Hence, we get $\vu(t)=-\vv(t), \forall t \geq 0$.
Thus, we obtain $\dot{s}(t)$ as follows:
\[ \dot{s}(t) \stackrel{\eqref{eq:grad-s}}{=} N s(t)^{2-\frac{2}{N}}\left(\vu(t)^\top\mZ\vv(t)-s(t)\right) =Ns(t)^{2-\frac{2}{N}}\left(-\vu(t)^\top\mZ\vu(t)-s(t)\right) \leq - Ns(t)^{3-\frac{2}{N}} \leq 0, \]
leading to $s(t) \to 0$. We could see that $\mW(t) \to \bm{0}$, which is a rank-unstable trajectory. Thus, the gradient flow of Eq.~\eqref{eq:obj} does not converges to a global minimizer. 

\paragraph{Remark.} We note that our Theorem \ref{thm:ab} covers this example by choosing $k=d$, i.e., the second case in Theorem \ref{thm:ab}.
Moreover, we can see $-s(0)\vu(0)\vu(0)^\top \not\in \gR_{b}(\mZ)$, since $-\vu(0)^\top\mZ\vu(0) \leq 0$. Thus, we could not further improve our definition of $\gR_{b}(\mZ)$ as well. 

\section{Convergence Rates to Global Minimizers}\label{sec:rates}
We briefly show the specific convergence rates in this section. 
We note that the proof of Theorem \ref{thm:ab} has already shown the rate for $\mW(t) \to \bm{0}$. Now we consider the rates to the global minimizers under Assumption \ref{ass}, which is common in previous works. That is, from Theorem \ref{thm:ab}, we consider the initialization which satisfies $a_1(0)+b_1(0) \neq 0 $. 
Typically, the trajectories can be divided into three stages\footnote{We only list the case $N \geq 3$, which is more common in practice. The theorem in this section also focus on the case $N\ge 3$. We leave the case of $N=2$ in Appendix \ref{case:N=2}, where we provide the similar results and convergence rates.}:

\paragraph{Stage 1.}
For $t \in [0, t_1]$, where $t_1 := \inf \{t: a_1(t)b_1(t) \ge0 \} < +\infty$, we have $a_1(t)b_1(t) \leq 0$, and the rates are
\[ 1-a_1(t)b_1(t) = \gO\left([(N-2)t]^{-\frac{c_1}{N-2}}\right), \ s(t)=\Omega\left([(N-2)t]^{-\frac{N}{N-2}}\right). \text{ [Theorem \ref{thm:stage1}] } \]

\paragraph{Stage 2.}
For $t \in (t_1, t_2]$, where $t_2:= \inf \{t: \vu(t)^\top\mZ\vv(t) \ge s(t)\}$, we have $a_1(t)b_1(t) > 0, \dot{s}(t)\le 0$, and
\[ 1-a_1(t)b_1(t) = \gO\left([(N-2)t]^{-\frac{c_2}{N-2}}\right), \ s(t)=\Omega\left([(N-2)t]^{-\frac{N}{N-2}}\right). \text{ [Theorem \ref{thm:stage2}] } \]

\paragraph{Stage 3.}
For $t \in (\max\{t_1, t_2\}, +\infty)$, we have $a_1(t)b_1(t) > 0$ and $\dot{s}(t) \geq 0$, and the rate are
\[ 1-a_1(t)b_1(t) = \gO\left(e^{-c_5t}\right), \quad  s_1-s(t) = \gO \left(e^{-\min\{c_5,c_6\}t}\right). \text{ [Theorem \ref{thm:stage3}] } \]

\begin{remark}
    We explain other minor cases: 1) If $t_1=0$, then the Stage 1 vanishes; 2) If $t_1\ge t_2$, then the Stage 2 vanishes;
    3) If $t_2=+\infty$\text{ [Theorem \ref{thm:stage2.5}] }, then we have similar rates as Stage 3: $1-a_1(t)b_1(t) = \gO(e^{-c_3t}), \ |s(t)-s_1| = \gO(e^{-c_4t})$. However, this case is not suitable in our framework.
\end{remark}

The convergence rates go through a polynomial to a linear rate, which is intuitively correct in practice.
Before we give the detail of analysis, we need some preparation in advance. Theorem \ref{thm:monotone} below shows the characteristic through $t_1$ and $t_2$. 

\begin{thm}\label{thm:monotone}
    Let $c_1(t):=a_1(t)b_1(t)$, $t_1:=\inf \{t: a_1(t)b_1(t) \ge0 \}$, $t_2:=\inf \{t: \vu(t)^\top\mZ\vv(t) \ge s(t) \}$. 
    Then we have
    (I) $t_1 < +\infty$ if $a_1(0)+b_1(0) \neq 0$, and $\dot{c}_1(t) \ge 0$ for all $t\ge 0$;
    (II) $\dot{s}(t)\le 0$ for $t\in[0,t_2)$, and $\dot{s}(t)\ge 0$ for $t\in[t_2,+\infty)$.
\end{thm}

\begin{remark}
    (I) in Theorem \ref{thm:monotone} tells us that the first stage, if exists, only appears a finite time in the beginning.
    (II) in Theorem \ref{thm:monotone} shows the induced weight norm ($\norm{\mW(t)}_F = s(t)$) goes through descending and ascending periods. If the initial induced weight norm starts with descending behavior, then it could descend forever, or it will change to ascending and continue increasing to $s_1$. If the initial induced weight norm begins with ascending behavior, then it would increase to $s_1$ directly. Such induced weight norm behavior also appears in deep linear networks with the logit loss \cite{lin2021faster}.
\end{remark}

\subsection{Convergence Rates of $s(t)$: Stage 1 and Stage 2}\label{sec:s-s12}
At Stage 1 and Stage 2, we have $\dot{s}(t) \leq 0$ from Theorem \ref{thm:monotone}. Now we first give a global lower bounds for the singular value $s(t)$, which may work as a proper lower bound within Stage 1 and Stage 2. Meanwhile, we provide a global upper bound of $s(t)$.

\begin{thm}\label{thm:s}
Assume $s(0)>0$, then we have $0<s(t)< s_0 := \max\{s_1,s(0)\}$ for all $t \geq 0$. Further we have
\begin{align}
    &\text{when } N = 2: s(t) \ge s(0) e^{-2(s_1+s_0)t}, i.e., s(t) = \Omega\left(e^{-2(s_1+s_0)t}\right); \label{eq:s-upper-2}\\
    &\text{when } N \geq 3: s(t) \geq \left[(s_1+s_0)(N-2)t+s(0)^{\frac{2}{N}-1}\right]^{-\frac{N}{N-2}}, i.e., s(t)=\Omega\left([(N-2)t]^{-\frac{N}{N-2}}\right). \label{eq:s-upper-3}
\end{align}
\end{thm}

We note that different lower bounds of $s(t)$ lead to different rates for the case $N=2$ and $N\geq 3$. For brevity, we only give the results for $N\geq 3$, and leave the simple case $N=2$ in Appendix \ref{case:N=2}.

\subsection{Convergence Rates of $a_1(t)b_1(t)$: Stage 1}
In the case where $a_1(0)b_1(0)<0$, we prove that the case will reduce to the case $a_1(0)b_1(0)\ge 0$ in a finite time when $a_1(0)+b_1(0)\neq 0$. We further give an upper bound for time staying in Stage 1 and a lower bound of $a_1(t)b_1(t)$.

\begin{thm}\label{thm:stage1}
	Suppose $N \ge 3$, $a_1(0)b_1(0)< 0$ and $a_1(0) + b_1(0) \neq 0$. Then we have 
	\[ 1-a_1(t)b_1(t) = \gO([(N-2)t]^{-\frac{c_1}{N-2}}),  0 \leq t \leq t_1, \]
	where $c_1 = (s_1+s_0)^{-1}$, and $s_0$ inherited from Theorem \ref{thm:s}. 
	Furthermore, we have the upper bound of $t_1$ below:
	\begin{equation}\label{eq:t1-upper}
	    t_1 \leq \frac{s(0)^{\frac{2}{N}-1}}{2(s_1+s_0)(N-2)} \cdot \left[ \left(\left|\frac{a_1(0)-b_1(0)}{a_1(0)+b_1(0)}\right|\right)^{(s_1+s_0)(N-2)} - 1 \right].
	\end{equation}
	Additionally, we could obtain
	\begin{equation}\label{eq:ab-lower-1}
	    a_1(t)b_1(t) = \Omega\left(\frac{(a_1(0)+b_1(0))^2}{N}\right), \text{ if } t \geq \frac{s(0)^{\frac{2}{N}-1}}{2(s_1+s_0)(N-2)} \cdot \left[ e\left(\left|\frac{a_1(0)-b_1(0)}{a_1(0)+b_1(0)}\right|\right)^{(s_1+s_0)(N-2)} - 1 \right].
	\end{equation}
\end{thm}

\begin{remark}
    The upper bound of $t_1$ in Theorem \ref{thm:stage1} shows that if $a_1(0)+b_1(0) \approx 0$, then the Stage 1 would last for a long time according to Eq.~\eqref{eq:t1-upper}. Moreover, Theorem \ref{thm:ab} has already shown that once $a_1(0)+b_1(0) = 0$, the trajectory would not converge to the global minimizer. 
    Hence, our finding in Theorem \ref{thm:stage1} is consistent with Theorem \ref{thm:ab}.
    Additionally, we also give guarantee of the trajectory to arrive at $\gG_b(\mZ)$ for some $b>0$ from Eq.~\eqref{eq:ab-lower-1}.
    That is, the trajectory enters in our global minimizer convergent set. 
\end{remark}

\subsection{Convergence Rates of $a_1(t)b_1(t)$: Stage 2}
Based on Theorem \ref{thm:stage1}, we can see after finite time, we obtain $a_1(t)b_1(t) > 0$, that is, the trajectory enters in the global minimizer convergent set $\gG_b(\mZ)$. 
In the following, we begin with $a_1(0)b_1(0) > 0$ for short. We discover a similar convergence rate in Stage 2.

\begin{thm}\label{thm:stage2}
	Assume $N\ge 3$, $a_1(0)b_1(0) > 0$. Then we have
	\begin{equation*}
	    1-a_1(t)b_1(t) = \gO\left([(N-2)t]^{-\frac{c_2}{N-2}}\right),
	\end{equation*}
	 where $c_2 = \frac{2(s_1-s_2)}{s_1+s_0}$, and $s_0$ inherited from Theorem \ref{thm:s}.
\end{thm}

\subsection{Convergence Rates of $a_1(t)b_1(t)$ and $s(t)$: Stage 3}
Before we start our analysis in Stage 3, we need to handle the minor case $t_2:=\inf \{t: \vu(t)^\top\mZ\vv(t) \ge s(t) \}= +\infty$.
We can assume $a_1(0)b_1(0)>0$ from Stage 1.

\begin{thm}\label{thm:stage2.5}
    Suppose $N \ge 3$, $a_1(0)b_1(0) > 0$ and $t_2=+\infty$. Then we have
    \[ 1-a_1(t)b_1(t)= \gO\left(e^{-c_3t}\right), |s(t)-s_1|=\gO\left(e^{-c_4 t}\right), \]
    where $c_3 = 2s_1^{1-\frac{2}{N}}\left(s_1-s_2\right)$, $c_4 =N s_1^{2-\frac{2}{N}}$.
\end{thm}

Now we turn to the case $t_2<+\infty$. Additionally, we can assume $\dot{s}(0) \geq 0$ for short in Stage 3.

\begin{thm}\label{thm:stage3}
	Assume $N\ge 3$, $a_1(0)b_1(0) > 0$, and $\dot{s}(0) \geq 0$. Then we have
	\begin{equation*}
	    1-a_1(t)b_1(t) = \gO(e^{-c_5t}), |s_1-s(t)|=\gO\left(e^{-\min\{c_5, c_6\}t} \right),
	\end{equation*}
    where $c_5 = 2s(0)^{1-\frac{2}{N}}\left(s_1-s_2\right)$, $c_6 =N s(0)^{2-\frac{2}{N}}$.
\end{thm}

As we mentioned before, the difference between the minor case $t_2=+\infty$ and Stage 3 is the constant above the exponent, and the proofs are similar between these two schemes. Thus, we combine them in a subsection.

\begin{remark}
    Though we don't provide an upper bound of $t_2$ here, we still have a slower global convergence guarantee of $\vu(t),\vv(t)$ following Stage 2.
    Moreover, we discover the linear rate in Stage 3 only appears in the late training phase from experiments (see Section \ref{sec:exp}), and gives high precision guarantee of solution at last.
    Furthermore, \citet{eftekhari2020training} also gave a linear rate in their restricted initialization set. Thus, we mainly focus on the previous stages to highlight that our results cover a larger initialization set.  
\end{remark}


\section{Experiments}\label{sec:exp}

\begin{figure}[t]
	\centering
	\begin{subfigure}[b]{0.48\textwidth}
		\includegraphics[width=\linewidth]{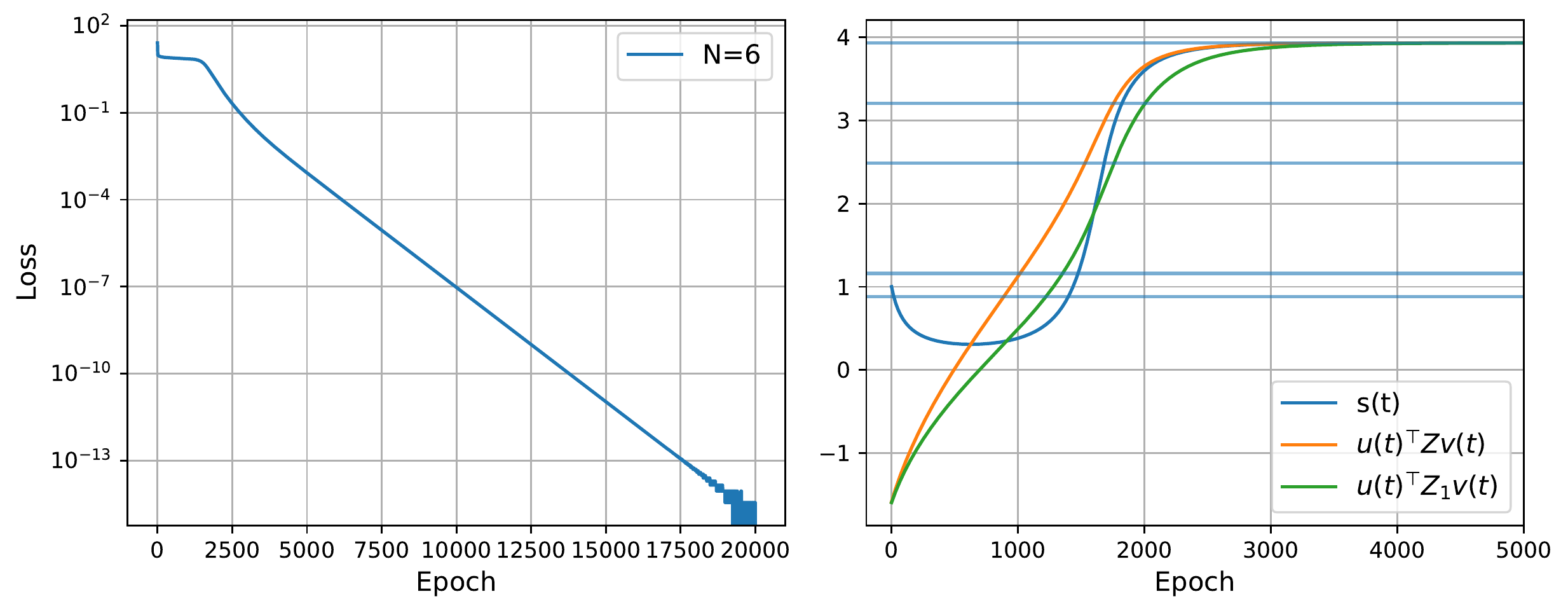}
		\caption{$k=0$.} \label{fig:k-0}
	\end{subfigure}
	\begin{subfigure}[b]{0.48\textwidth}
		\includegraphics[width=\linewidth]{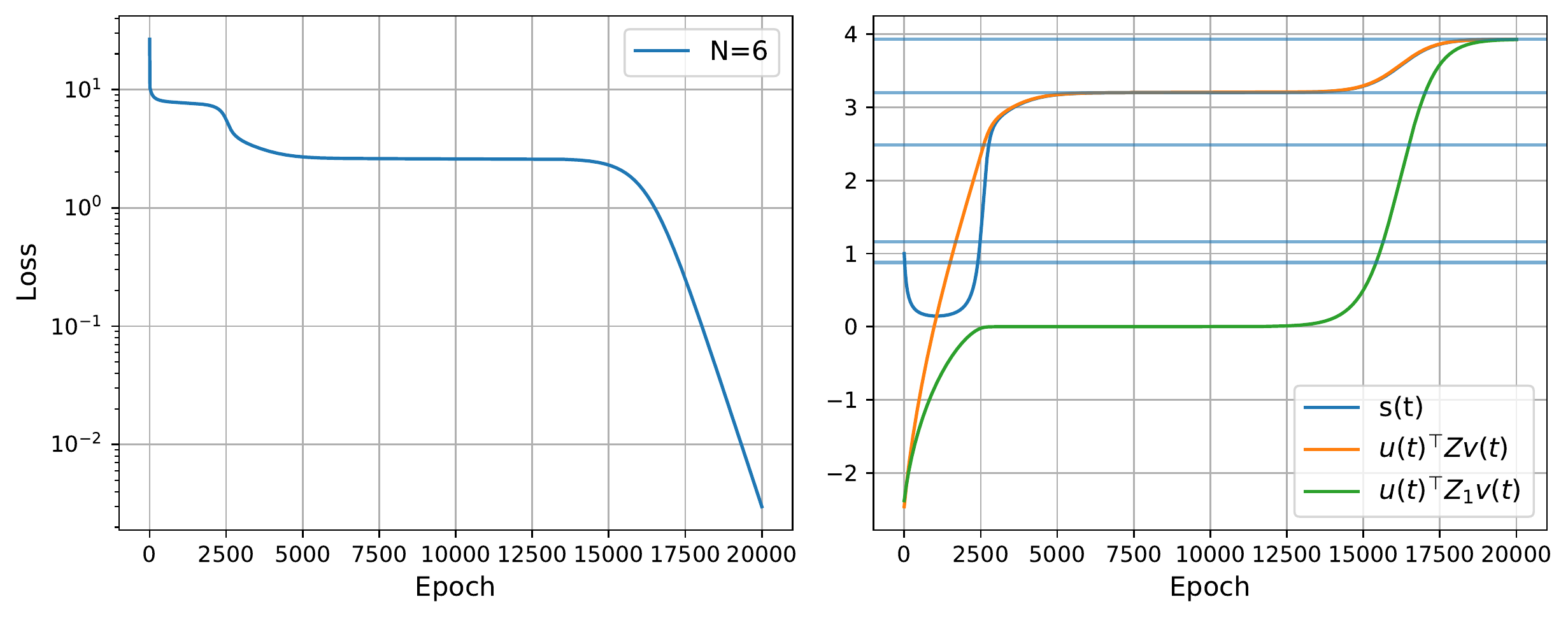}
		\caption{$k=1$.} \label{fig:k-1}
	\end{subfigure}
	\begin{subfigure}[b]{0.48\textwidth}
		\includegraphics[width=\linewidth]{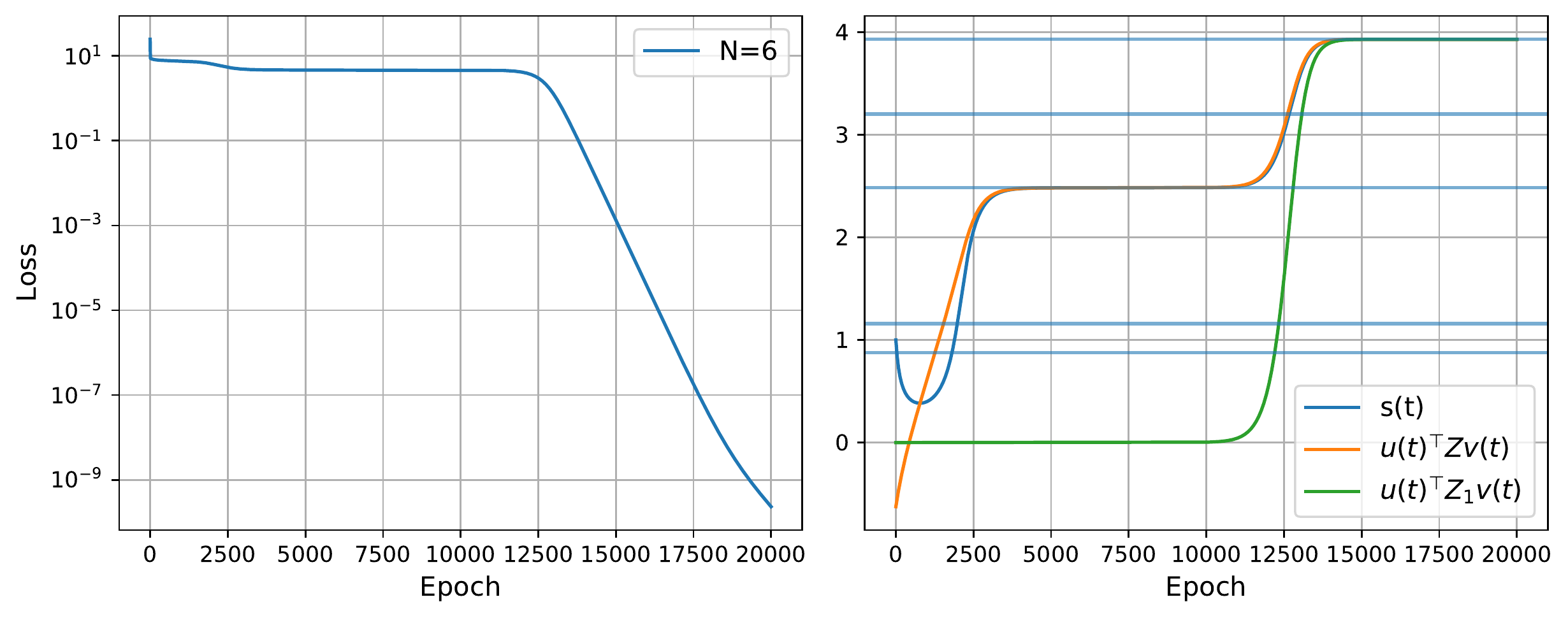}
		\caption{$k=2$.} \label{fig:k-2}
	\end{subfigure}
	\begin{subfigure}[b]{0.48\textwidth}
		\includegraphics[width=\linewidth]{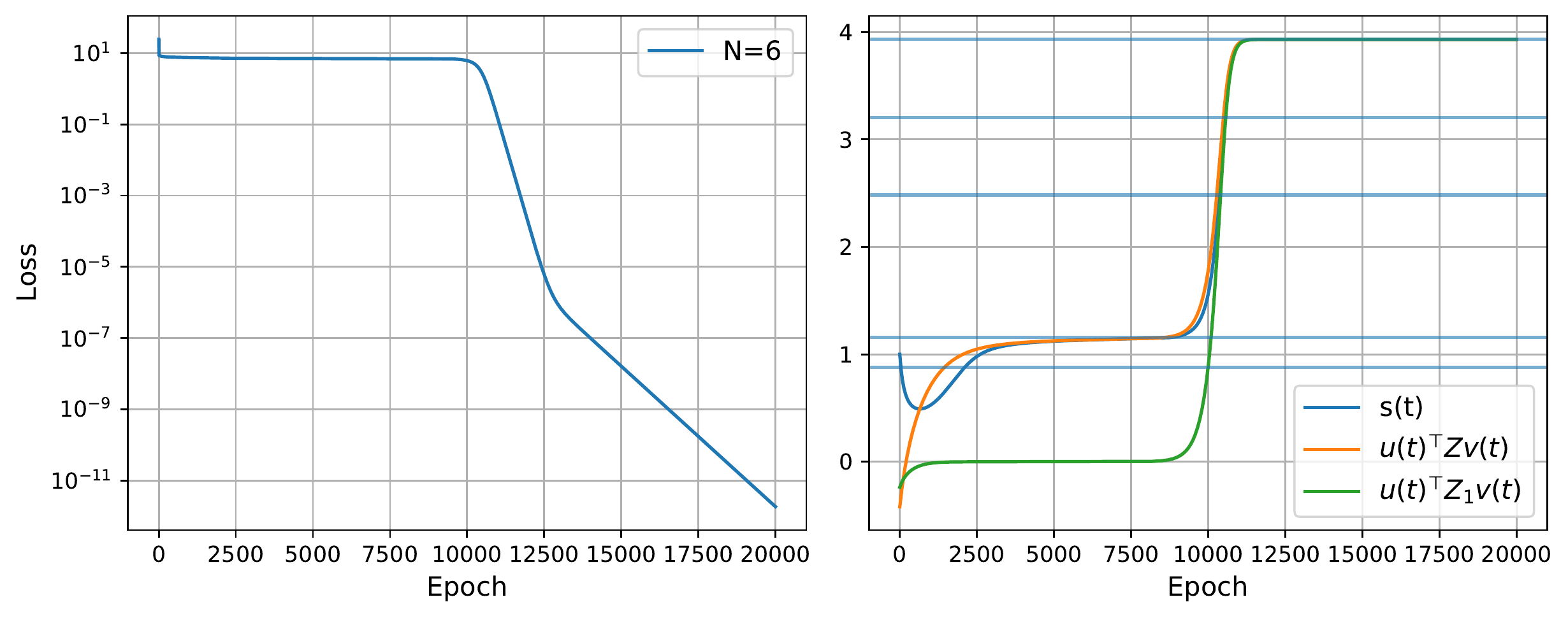}
		\caption{$k=3$.} \label{fig:k-3}
	\end{subfigure}
	\begin{subfigure}[b]{0.48\textwidth}
		\includegraphics[width=\linewidth]{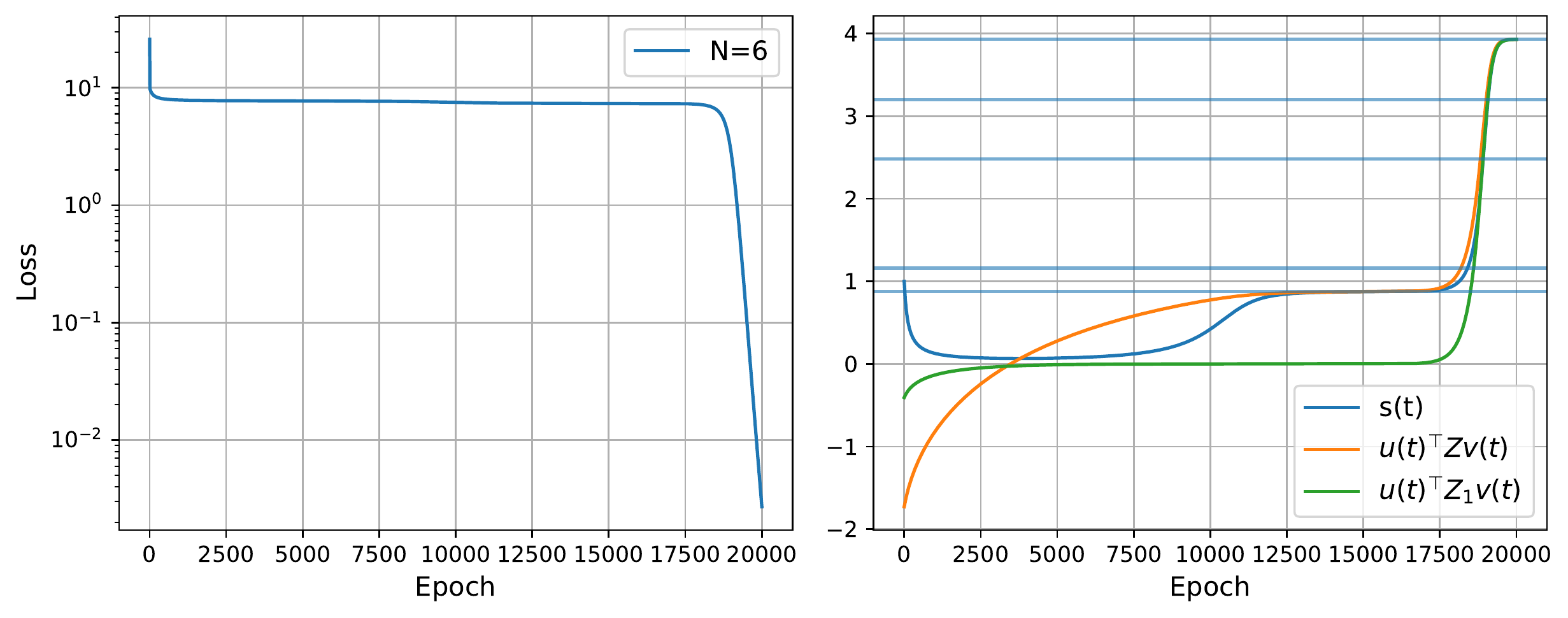}
		\caption{$k=4$.} \label{fig:k-4}
	\end{subfigure}
	\begin{subfigure}[b]{0.48\textwidth}
		\includegraphics[width=\linewidth]{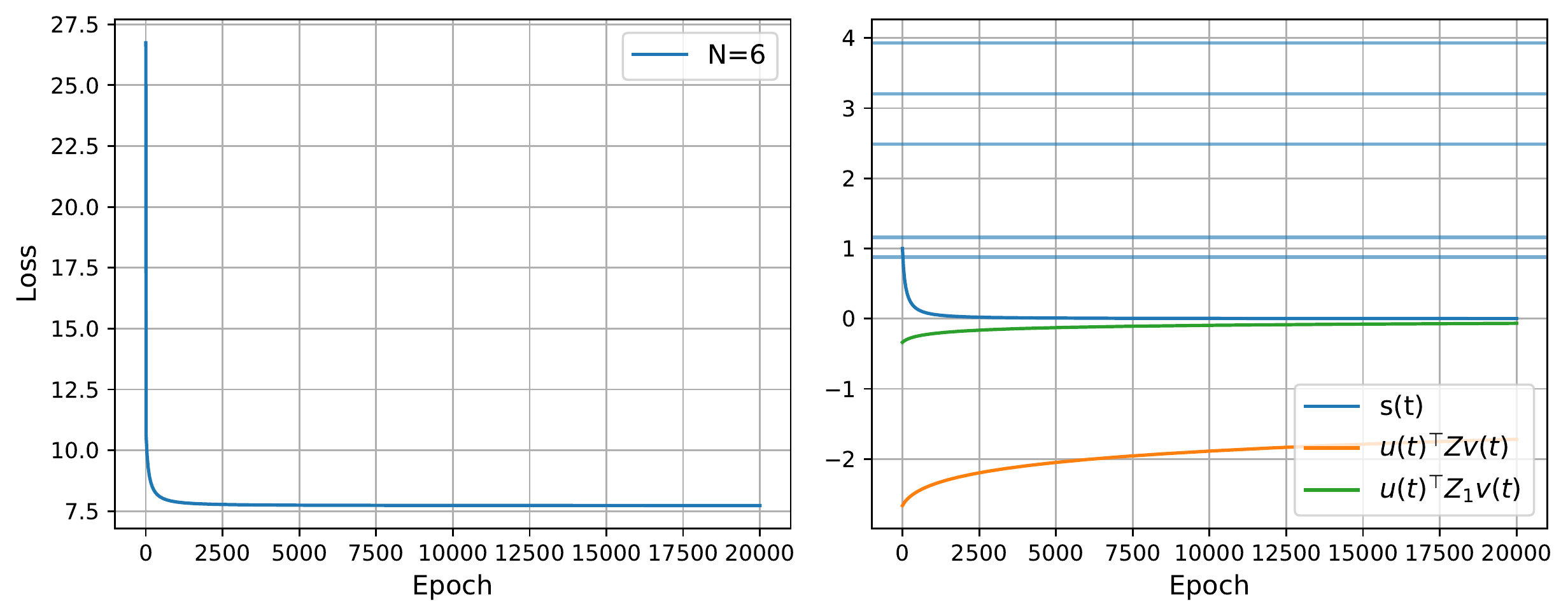}
		\caption{$k=5$.} \label{fig:k-5}
	\end{subfigure}
	\caption{We choose $N=6, d=5$ with hidden-layer width $(d_N, \dots, d_0) = (5, 4, 1, 10, 5, 3, 8)$, and set different $k \in [0, d]$ in Theorem \ref{thm:ab}. The transparent horizontal lines are the singular value of $\mZ$ in order. Learning rate is $5\times 10^{-4}$. Legends present $s(t), \vu(t)^\top \mZ \vv(t), \vu(t)^\top \mZ_1 \vv(t)$. Here $t$ is the running step in gradient descent.}
	\label{fig:deep-linear-k}
\end{figure}

In this section, we conduct simple numerical experiments to verify our discovery.
\paragraph{Different kinds of trajectories.} We construct $\vu(0) = \mU \bm{\alpha}_1$ and $\vv(0) = \mV \bm{\alpha}_2$, where $\bm{\alpha}_1 \in \sR^{d_y \times 1}$ and $\bm{\alpha}_2 \in \sR^{d_x \times 1}$ have the inverse items until $k$-th dimension, i.e., $(\bm{\alpha}_1)_i + (\bm{\alpha}_2)_i=0, \forall i \in [k], (\bm{\alpha}_1)_{k+1} + (\bm{\alpha}_2)_{k+1} \neq 0$, and $k \leq d$. 
Then we can see $a_i(0)+b_i(0) = \vu_i^\top\vu(0)+\vv_i^\top\vv(0) = 0, \forall i \in [k], a_{k+1}(0)+b_{k+1}(0) \neq 0$.
After $\vu(0), \vv(0)$ decided, we construct $\mW_i(0) = \vh_{i+1}\vh_{i}^\top$ with $\norm{\vh_i} = 1$ and $\vh_{1} = \vv(0), \vh_{N+1} = \vu(0)$ to obtain a balanced initialization $(\mW_1(0), \dots, \mW_N(0))$ and $\mW(0) = \vu(0)\vv(0)^\top$.
Finally, we run gradient descent (GD) for the problem \eqref{eq:obj} with a small learning rate $5\times 10^{-4}$.
The simulations are shown in Figure \ref{fig:deep-linear-k}.

As Figure \ref{fig:deep-linear-k} depicts, we discover $\vu(t)^\top\mZ\vv(t), \vu(t)^\top\mZ_1\vv(t)$ are non-decreasing as our Proposition \ref{prop:base} shows.
Moreover, we could see our construction gives a stuck region around $s_{k+1}$ according to the choice of $k \leq d$. Though our Theorem \ref{thm:ab} shows that the gradient flow of $\mW(t)$ would finally converges $s_{k+1}\vu_{k+1}\vv_{k+1}^\top$, we find after a period of (long) time, gradient descent can escape from the saddle point around $s_{k+1}\vu_{k+1}\vv_{k+1}^\top$, and finally converges to a global minimizer. We consider the numerical error during optimization and unbalanced weight matrix caused by GD may lead to the inconsistent of gradient flow and its discrete version GD. 
Overall, we describe the possible convergence behavior of all initialization in the ideal setting.

\begin{figure}[t]
	\centering
	\includegraphics[width=0.9\linewidth]{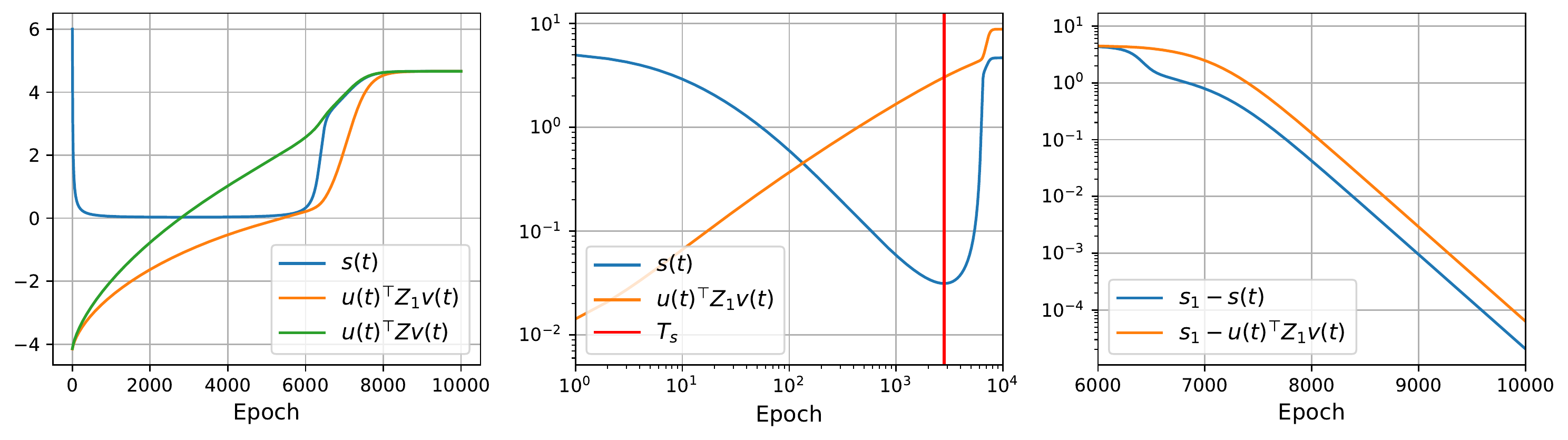}
	\caption{We plot choose $N=6, d=5$ with hidden-layer width $(d_N, \dots, d_0) = (5, 4, 1, 10, 5, 3, 8)$, and set $k=0$ in Theorem \ref{thm:ab}, i.e., the trajectory converges to the global minimizer. Learning rate is $5\times 10^{-4}$. 
	Left: Variation of $s(t), \vu(t)^\top \mZ \vv(t), \vu(t)^\top \mZ_1 \vv(t)$ during the whole optimization. 
	Middle: Polynomial convergence of $s(t)$ and $\vu(t)^\top \mZ_1 \vv(t)$ in the second stage. We plot $\vu(t)^\top \mZ_1 \vv(t)-\vu(0)^\top \mZ_1 \vv(0)$ instead of $\vu(t)^\top \mZ_1 \vv(t)$ to avoid negative values. Here, $T_s$ is the time that the monotonicity of $s(t)$ changes. 
	Right: Linear convergence of $s(t)$ and $\vu(t)^\top \mZ_1 \vv(t)$ in the final stage.
	Here $t$ is the running step in gradient descent.}
	\label{fig:deep-linear-detail-0}
\end{figure}

\paragraph{Trajectories converging to the global minimizer.}
We also plot the trajectory converging to the global minimizer in detail shown Figure \ref{fig:deep-linear-detail-0}. 
To give a more clear variation of stages, we adopt $s(0)=5$ and a small $\vu(0)^\top\mZ_1\vv(0)$.
As the left figure of Figure \ref{fig:deep-linear-detail-0} shown, $s(t)$ first decreases, then increases. 
Additionally, the middle figure shows that $s(t)$ decreases and $\vu(t)^\top \mZ_1 \vv(t)$ increases with an approximate polynomial rate (noting the log scale in both x-axis and y-axis). 
Moreover, the middle figure also shows that once $s(t)$ increases, i.e., $t \geq T_s$, $\vu(t)^\top \mZ_1 \vv(t)$ will increase much faster, and switches to another stage as we prove.
Finally, we observe the final stage, that is, the linear convergence of both $s(t)$ and $\vu(t)^\top \mZ_1 \vv(t)$ to $s_1$ in the right graph of Figure \ref{fig:deep-linear-detail-0}.
Overall, we conclude our convergent rates match with the numerical experiments well.

\section{Conclusion}\label{sec:con}
In this work, we have studied the training dynamic of deep linear networks which have a one-neuron layer. 
Specifically, we focus on the gradient flow methods under the quadratic loss and balanced initialization. 
We have shown the convergent point of an arbitrary starting point.
Moreover, we also give the convergence rates of the trajectories towards the global minimizer by stages.
The behavior predicted by our theorems is also observed in numerical experiments.
However, the analysis of general linear networks without a one-neuron layer remains a challenging open problem.
We hope that our limited view of training trajectories would bring a better understanding of (linear) neural networks. 

\bibliography{reference}
\bibliographystyle{plainnat}

\appendix
\section{Auxiliary Conclusion}
\subsection{Previous Results}
\begin{lemma}[Lemma 3.3 in \citet{eftekhari2020training}]\label{aux-lemma:1}
   For the induced flow in Eq.~\eqref{eq:induced}, we have that $\mathrm{rank}(\mW(t)) = \mathrm{rank}(\mW(0)), \forall t \geq 0$, provided that $\mX\mX^\top$ is invertible and the network depth $N \geq 2$.
\end{lemma}

\begin{lemma}[Lemma 4 in \citet{arora2019implicit}]\label{thm:aux-2}
    Let $\alpha \geq 1/2$ and $g: [0, \infty) \to \sR$ be a continuous function. Consider the initial value problem:
    \[ s(0) =s_0, \dot{s}(t) = (s^2(t))^{\alpha} \cdot g(t), \forall t \geq 0, \]
    where $s_0 \in \sR$. Then, as long as it does not diverge to $\pm \infty$, the solution to this problem ($s(t)$) has the same sign as its initial value ($s0$). That is, $s(t)$ is identically zero if $s_0 = 0$, is positive if $s_0 > 0$, and is negative if $s_0 < 0$.
\end{lemma}

\begin{thm}[Theorem 5 in \citet{bah2021}]\label{thm:aux-1}
    Assume $\mX\mX^\top$ has full rank. Then the flows $\mW_i(t)$ defined by Eq.~\eqref{eq:flow-wi} and $\mW(t)$ given by Eq.~\eqref{eq:induced} are defined and bounded for all $t \geq 0$ and $(\mW_1,\dots, \mW_N)$ converges to a critical point of $L^N$ as $t \to +\infty$.
\end{thm}

\begin{defn}[Definition 27 in \citet{bah2021}]
    Let $(\gM, g)$ be a Riemannian manifold with Levi-Civita connection $\nabla$ and let $f : \gM \to \sR$
    be a twice continuously differentiable function. 
    A critical point $x_0 \in \gM$, i.e., $\nabla^g f(x_0) = 0$ is called a strict saddle point if Hess $f(x)$ has a negative eigenvalue. We denote the set of all strict saddles of $f$ by $\gX = \gX(f)$.
    We say that $f$ has the strict saddle point property, if all critical points of f that are not local minima are
    strict saddle points.
\end{defn}

The following theorem shows that flows avoid strict saddle points almost surely (See Section 6.2 in \citet{bah2021} for detail).

\begin{thm}[Theorem 28 in \citet{bah2021}]\label{thm:aux-3}
Let $L : \gM \to \sR$ be a $C^2$-function on a second countable finite dimensional Riemannian
manifold $(\gM, g)$, where we assume that $\gM$ is of class $C^2$ as a manifold and the metric $g$ is of class $C^1$.
Assume that $\phi_t(x_0)$ exists for all $x_0 \in \gM$ and all $t \in [0, +\infty)$. Then the set
\[ \gS_L := \{x_0\in\gM : \lim_{t \to +\infty} \phi_t(x_0) \in\gX = \gX(L)\} \]
of initial points such that the corresponding flow converges to a strict saddle point of $L$ has measure zero.
\end{thm}

\begin{prop}[Proposition 33 in \citet{bah2021}]\label{thm:aux-4}
    The function $L^1$ on $\gM_k$ for $k \leq r$ satisfies the strict saddle point property. More precisely, all critical points of $L^1$ on $\gM_k$ except for the global minimizers are strict saddle points.
\end{prop}

\subsection{Auxiliary Lemmas}
\begin{lemma}[Dynamic of $s(t),\vu(t), \vv(t)$]\label{lemma:dot-s-u-v}
    We give the derivation of $\dot{\vu}(t), \dot{\vv}(t), \dot{s}(t)$ shown in the main context in this lemma:
    \begin{align*}
        \dot{\vu}(t) &= s(t)^{1-\frac{2}{N}}\left(\mI_{d_y}-\vu(t)\vu(t)^\top\right)\mZ\vv(t),  \\
        \dot{\vv}(t) &= s(t)^{1-\frac{2}{N}}\left(\mI_{d_x}-\vv(t)\vv(t)^\top\right)\mZ^\top\vu(t),\\
        \dot{s}(t) &= N s(t)^{2-\frac{2}{N}}\left(\vu(t)^\top\mZ\vv(t)-s(t)\right).
    \end{align*}
\end{lemma}
\begin{proof}
    $\dot{s}(t)$ directly follows \citet[Theorem 3]{arora2019implicit}. 
    As for $\dot{\vu}(t)$ and $\dot{\vv}(t)$, we begin with $\vu(t)^\top\vu(t) = \vv(t)^\top\vv(t) = 1$. Then by taking the derivative of the identities, we get
    \begin{equation}\label{eq:uu-vv}
        \vu(t)^\top\dot{\vu}(t)=\vv(t)^\top\dot{\vv}(t)=0, \forall t\geq 0.
    \end{equation}
    By taking derivative of both sides of the SVD $\mW(t)=s(t)\vu(t)\vv(t)^\top$, we also find that
    \[ \dot{\mW}(t) = s(t)\dot{\vu}(t)\vv(t)^\top + s(t)\vu(t)\dot{\vv}(t)^\top+\dot{s}(t)\vu(t)\vv(t)^\top, \forall t \geq 0. \]
    Hence, multiplying $\left(\mI_{d_y}-\vu(t)\vu(t)^\top\right)$ and $\vv(t)$, we get
    \[ s(t)^{-1} \left(\mI_{d_y}-\vu(t)\vu(t)^\top\right) \dot{\mW}(t) \vv(t) = \left(\mI_{d_y}-\vu(t)\vu(t)^\top\right)\dot{\vu}(t). \]
    From Eq.~\eqref{eq:uu-vv}, we know $\dot{\vu}(t) \perp \vu(t)$. Therefore, we obtain
    \begin{equation}\label{eq:dot-u}
        \dot{\vu}(t) = s(t)^{-1} \left(\mI_{d_y}-\vu(t)\vu(t)^\top\right) \dot{\mW}(t) \vv(t).
    \end{equation}
    Similarly, we can find that
    \begin{equation}\label{eq:dot-v}
        \dot{\vv}(t) = s(t)^{-1} \left(\mI_{d_x}-\vv(t)\vv(t)^\top\right) \dot{\mW}(t)^\top \vu(t).
    \end{equation}
    Now we replace $\dot{\mW}(t)$ by Eq.~\eqref{eq:induced} and $\mW(t)=s(t)\vu(t)\vv(t)^\top$:
    \begin{align*}
	\dot{\mW}(t) \stackrel{\eqref{eq:induced}}{=} & -N s(t)^{2-\frac{2}{N}} \vu(t)\vu(t)^\top \left[\mW(t)-\mZ\right]\vv(t)\vv(t)^\top \\
	&- s(t)^{2-\frac{2}{N}} \left(\mI_{d_y}-\vu(t)\vu(t)^\top\right) \left[\mW(t)-\mZ\right]\vv(t)\vv(t)^\top \\
	&- s(t)^{2-\frac{2}{N}}\vu(t)\vu(t)^\top \left[\mW(t)-\mZ\right]\left(\mI_{d_x}-\vv(t)\vv(t)^\top\right) \\
	=& -N s(t)^{1-\frac{2}{N}} \left(s(t)-\vu(t)^\top \mZ \vv(t) \right)\mW(t)\\
	&+ s(t)^{2-\frac{2}{N}} \left(\mI_{d_y}-\vu(t)\vu(t)^\top\right) \mZ\vv(t)\vv(t)^\top \\
	&+ s(t)^{2-\frac{2}{N}} \vu(t)\vu(t)^\top \mZ\left(\mI_{d_x}-\vv(t)\vv(t)^\top\right).
	\end{align*}
	Substituting $\dot{\mW}(t)$ back into Eq.~\eqref{eq:dot-u} and \eqref{eq:dot-v}, we reach 
	\[ \dot{\vu}(t) = s(t)^{1-\frac{2}{N}}\left(\mI_{d_y}-\vu(t)\vu(t)^\top\right)\mZ\vv(t), \ \dot{\vv}(t) = s(t)^{1-\frac{2}{N}}\left(\mI_{d_x}-\vv(t)\vv(t)^\top\right)\mZ^\top\vu(t). \]
\end{proof}

\begin{prop}[Stationary Singular Vector]\label{prop:station-vector}
	If for time $T \geq 0$, $s(T)>0, \dot{\vu}(T) = \bm{0}, \dot{\bm{\vv}}(T) = \bm{0}$, then
	\begin{equation*}
		\vu(T) = \pm \vu_i, \vv(T) = \pm \vv_i, \mathrm{for \ some} \ i \leq d, \mathrm{ or } \ \vu(T) \perp \vu_i, \vv(T) \perp \vv_i, \forall i \leq d.
	\end{equation*}
	Moreover, $\dot{\vu}(t) = \bm{0}, \dot{\bm{\vv}}(t) = \bm{0}, \forall t \geq T$, that is, $\vu(t) = \vu(T), \vv(t) = \vv(T), \forall t \geq T$.
\end{prop}
\begin{proof}
	From $s(T)>0, \dot{\vu}(T) = \bm{0}, \dot{\bm{\vv}}(T) = \bm{0}$ and Eqs.~\eqref{eq:grad-u} and \eqref{eq:grad-v}, we obtain
	\begin{equation}\label{eq:sat-uv}
	\left(\mI_{d_y}-\vu(T)\vu(T)^\top\right)\mZ\vv(T)=\bm{0}, \  \left(\mI_{d_x}-\vv(T)\vv(T)^\top\right)\mZ^\top\vu(T)=\bm{0}.
	\end{equation}
	Hence, we could see
	\begin{equation}\label{eq:sta}
	\mZ\vv(T) = \vu(T)^\top\mZ\vv(T) \cdot \vu(T), \mZ^\top\vu(T)=\vv(T)^\top\mZ^\top\vu(T) \cdot \vv(T),
	\end{equation}
	showing that $\mZ^\top\mZ\vv(T) = \left[\vu(T)^\top\mZ\vv(T)\right]^2 \cdot \vv(T)$, $\mZ\mZ^\top\vu(T) = \left[\vu(T)^\top\mZ\vv(T)\right]^2 \cdot \vu(T)$. Thus, we can see $\vu(T), \vv(T)$ are the eigenvectors of 
	$\mZ^\top\mZ, \mZ\mZ^\top$ with the same eigenvalue $\left[\vu(T)^\top\mZ\vv(T)\right]^2$. 
	Therefore, if $\vu(T)^\top\mZ\vv(T) \neq 0$, we obtain $\vu(T) = \pm\vu_i, \vv(T) = \pm\vv_i, \mathrm{for \ some} \ i \in [d]$. Otherwise, $\vu(T)^\top\mZ\vv(T) = 0$. From Eq.~\eqref{eq:sta}, we obtain $\mZ\vv(T)=\bm{0}, \mZ^\top\vu(T)=\bm{0}$, showing that $\vu(T) \perp \vu_i, \vv(T) \perp \vv_i, \forall i \in [d]$.
	
	Finally, we note that variation of $s(t)$ can not make $\dot{\vu}(t), \dot{\bm{\vv}}(t)$ become nonzero from time $T$. Specifically, 
	\begin{align*}
	\dot{\mW}(T) \stackrel{\eqref{eq:induced}}{=} & -N s(T)^{2-\frac{2}{N}} \vu(T)\vu(T)^\top \left[\mW(T)-\mZ\right]\vv(T)\vv(T)^\top \\
	&- s(T)^{2-\frac{2}{N}} \left(\mI_{d_y}-\vu(T)\vu(T)^\top\right) \left[\mW(T)-\mZ\right]\vv(T)\vv(T)^\top \\
	&- s(T)^{2-\frac{2}{N}} \vu(T)\vu(T)^\top \left[\mW(T)-\mZ\right]\left(\mI_{d_x}-\vv(T)\vv(T)^\top\right) \\
	\stackrel{\eqref{eq:sat-uv}}{=}& -N s(T)^{2-\frac{2}{N}} \vu(T)\vu(T)^\top \left[\mW(T)-\mZ\right]\vv(T)\vv(T)^\top \\
	\ = \ & -N s(T)^{2-\frac{2}{N}} \left(s(T)-\vu(T)^\top \mZ \vv(T)\right) \cdot \vu(T)\vv(T)^\top.
	\end{align*}
	Therefore, once $\dot{\vu}(T) = \bm{0}, \dot{\bm{\vv}}(T) = \bm{0}$, then $\vu(t) = \vu(T), \vv(t) = \vv(T), \forall t \geq T$.
\end{proof}

\begin{lemma}\label{lemma:s(t)-two}
    If for certain $t \geq 0$, $\vu(t)^\top\mZ\vv(t) = s(t)>0$, and $\dot{\vu}(t) \neq \bm{0}$ or $\dot{\vv}(t) \neq \bm{0}$, then we have
    \[ \frac{d\left(\vu(t)^\top\mZ\vv(t)-s(t)\right)}{dt} > 0. \]
\end{lemma}

\begin{proof}
    Since $\dot{\vu}(t) \neq \bm{0}$ or $\dot{\vv}(t) \neq \bm{0}$, and $s(t)>0$, by Eqs.~\eqref{eq:grad-u} and \eqref{eq:grad-v}, we obtain 
    \begin{equation}\label{eq:geq0}
        \norm{\left(\mI_{d_y}-\vu(t)\vu(t)^\top\right)\mZ\vv(t)}_2^2 +\norm{\left(\mI_{d_x}-\vv(t)\vv(t)^\top\right)\mZ^\top\vu(t)}_2^2 > 0.
    \end{equation}
    From the derivation of $d(\vu(t)^\top\mZ\vv(t)-s(t)) / d t$ and $d s(t) / d t$, we get
    \begin{equation*}
        \begin{aligned}
            \lefteqn{\frac{d\left(\vu(t)^\top\mZ\vv(t)-s(t)\right)}{d t}}\\
            &\stackrel{\substack{\eqref{eq:grad-s} \\ \eqref{eq:mono-uzv}}}{=} s(t)^{1-\frac{2}{N}} \bigg[\norm{\left(\mI_{d_y}-\vu(t)\vu(t)^\top\right)\mZ\vv(t)}_2^2 +\norm{\left(\mI_{d_x}-\vv(t)\vv(t)^\top\right)\mZ^\top\vu(t)}_2^2 - N s(t)\left(\vu(t)^\top\mZ\vv(t)-s(t)\right)\bigg] \\
            &\ = \ s(t)^{1-\frac{2}{N}} \bigg[\norm{\left(\mI_{d_y}-\vu(t)\vu(t)^\top\right)\mZ\vv(t)}_2^2 +\norm{\left(\mI_{d_x}-\vv(t)\vv(t)^\top\right)\mZ^\top\vu(t)}_2^2 \bigg] \stackrel{\eqref{eq:geq0}}{>}0, 
        \end{aligned}
    \end{equation*}
    where the second equality uses the assumption $\vu(t)^\top\mZ\vv(t) = s(t)$.
\end{proof}

\section{Missing Proofs}

\subsection{Proof of Proposition \ref{prop:base}}
\begin{proof}
	1). From \citet[Theorem 5]{bah2021} (Theorem \ref{thm:aux-1}), we have $\mW(t)$ converges. Thus $s(t) = \norm{\mW(t)}_F$ also converges, and not diverges to infinity. Applying \citet[Lemma 4]{arora2019implicit} (Lemma \ref{thm:aux-2}), we can see $s(t)$ obviously preserves the sign of its initial value.
	
	2). $\vu(t)^\top\mZ\vv(t)$ is non-decreasing follows
	\begin{equation}\label{eq:mono-uzv}
	\begin{aligned}
	& \frac{d \vu(t)^\top\mZ\vv(t)}{d t} = \frac{d \vu(t)^\top}{d t}\mZ\vv(t) + \vu(t)^\top\mZ\frac{d \vv(t)}{d t} \\
	\stackrel{\eqref{eq:grad-u}, \eqref{eq:grad-v}}{=} & s(t)^{1-\frac{2}{N}} \left[\norm{\left(\mI_{d_y}-\vu(t)\vu(t)^\top\right)\mZ\vv(t)}_2^2 +\norm{\left(\mI_{d_x}-\vv(t)\vv(t)^\top\right)\mZ^\top\vu(t)}_2^2 \right] \geq 0.
	\end{aligned}
	\end{equation}
	Additionally, since $\norm{\vu(t)} = \norm{\vv(t)} = 1$, we have $\vu(t)^\top\mZ\vv(t) \leq s_1$. Hence, $\vu(t)^\top\mZ\vv(t)$ converges.
	
	3). Using Eq.~\eqref{eq:grad-ab}, we obtain
	\begin{equation*}
	\begin{aligned}
	\frac{d a_1(t)b_1(t)}{d t} &= \frac{d a_1(t)}{d t} \cdot b_1(t) + a_1(t) \cdot \frac{d b_1(t)}{d t} \stackrel{\eqref{eq:grad-ab}}{=} s(t)^{1-\frac{2}{N}}\bigg(s_1b_1^2(t) + s_1a_1^2(t)-2a_1(t)b_1(t) \sum_{j=1}^d\left[s_j a_j(t)b_j(t)\right]\bigg)\\
	&\geq s(t)^{1-\frac{2}{N}}\bigg(s_1b_1^2(t)+s_1a_1^2(t)-2s_1 |a_1(t)b_1(t)| \sum_{j=1}^d \left|a_j(t)b_j(t)\right| \bigg)\\
	&\geq s(t)^{1-\frac{2}{N}}\bigg(s_1b_1^2(t)+s_1a_1^2(t)-2s_1 |a_1(t)b_1(t)| \bigg) = s_1s(t)^{1-\frac{2}{N}}\left(|b_1(t)|-|a_1(t)|\right)^2\ge 0,
	\end{aligned}
	\end{equation*}
	where the second inequality uses Cauchy inequality:
	\begin{equation*}
	\bigg(\sum_{j=1}^d \left|a_j(t)b_j(t)\right|\bigg)^2 \leq \bigg(\sum_{j=1}^d a^2_j(t)\bigg) \cdot \bigg(\sum_{j=1}^d b^2_j(t)\bigg) \leq \bigg(\sum_{j=1}^{d_y} a^2_j(t)\bigg) \cdot \bigg(\sum_{j=1}^{d_x} b^2_j(t)\bigg) = 1.
	\end{equation*}
	
 	We note that $s_1 a_1(t) b_1(t) = s_1 \cdot \vu(t)^\top \vu_1 \cdot \vv_1^\top\vv(t) = \vu(t)^\top \left(s_1\vu_1 \vv_1^\top\right)\vv(t) =\vu(t)^\top\mZ_1 \vv(t)$. Hence, we obtain $\vu(t)^\top\mZ_1\vv(t)$ is non-decreasing.
	Moreover, since $\norm{\vu(t)} = \norm{\vv(t)} = 1$, we have $\vu(t)^\top\mZ_1\vv(t) \leq s_1$. Hence, $\vu(t)^\top\mZ_1\vv(t)$ also converges.
	
	4). Using the derivation in the above, we obtain 
	\begin{equation}\label{eq:aibi}
	\begin{aligned}
	\frac{d \left(a_i(t)+b_i(t)\right)}{dt} & \stackrel{\eqref{eq:grad-ab}}{=} s(t)^{1-\frac{2}{N}}\bigg[s_i\left(b_i(t) + a_i(t)\right)-\left(a_i(t)+b_i(t)\right) \sum_{j=1}^d\left[s_ja_j(t)b_j(t)\right]\bigg] \\
	&= s(t)^{1-\frac{2}{N}}\left(a_i(t) + b_i(t)\right)\bigg(s_i-\sum_{j=1}^d\left[s_ja_j(t)b_j(t)\right]\bigg).
	\end{aligned}
	\end{equation}
	Moreover, $|a_{i}(t)+ b_{i}(t)| = |\vv_{i}^\top\vu(t)+\vv_{i}^\top\vv(t)|\leq 2$, showing that $a_i(t)+b_i(t)$ does not diverge to infinity.
	Hence, by \citet[Lemma 4]{arora2019implicit} (Lemma \ref{thm:aux-2}), $a_i(t)+b_i(t)$ obviously preserves the sign of its initial value.
	
	5). Since $a_i(0)+b_i(0) = 0$, we get $a_i(t)+b_i(t) = 0$ by 4), i.e., 
	\begin{equation}\label{eq:ab-relat}
	a_i(t)= -b_i(t), \forall i\in[k], t \geq 0. 
	\end{equation} 
	Now we can bound
	\begin{equation}\label{eq:ineq2}
	\begin{aligned}
	\sum_{j=k+1}^d s_ja_j(t)b_j(t) & \leq s_{k+1} \sum_{j=k+1}^d |a_j(t)b_j(t)| \leq s_{k+1} \sqrt{\sum_{j=k+1}^d a_j^2(t) \sum_{j=k+1}^d b_j^2(t)} \\
	&\leq s_{k+1}\sqrt{\bigg(1-\sum_{j=1}^k a_j^2(t)\bigg) \cdot \bigg(1-\sum_{j=1}^k b_j^2(t)\bigg)} \stackrel{\eqref{eq:ab-relat}}{=} s_{k+1}\bigg(1-\sum_{j=1}^k a_j^2(t)\bigg).
	\end{aligned}	
	\end{equation}
	Hence, we obtain 
	\begin{equation}\label{eq:sk}
	\begin{aligned}
	& \qquad s_{k+1}-\sum_{j=1}^d\left[s_j a_j(t)b_j(t)\right] \stackrel{\eqref{eq:ab-relat}}{=} s_{k+1}+\sum_{j=1}^{k}s_j a_j^2(t)-\sum_{j=k+1}^d\left[s_j a_j(t)b_j(t)\right] \\
	&\stackrel{\eqref{eq:ineq2}}{\geq} s_{k+1} \bigg(1+\sum_{j=1}^k a_j^2(t) \bigg)-s_{k+1} \bigg(1-\sum_{j=1}^k a_j^2(t) \bigg) = 2s_{k+1}\sum_{j=1}^k a_j^2(t) \geq 0.
	\end{aligned}
	\end{equation}
	Now we consider the gradient of $a_{k+1}(t)+b_{k+1}(t)$:
	\begin{equation}\label{eq:s-sab}
	    \frac{d \left(a_{k+1}(t)+b_{k+1}(t)\right)}{d t} \stackrel{\eqref{eq:aibi}}{=} s(t)^{1-\frac{2}{N}}\left(a_{k+1}(t) + b_{k+1}(t)\right)\bigg(s_{k+1}-\sum_{j=1}^d\left[s_j a_j(t)b_j(t)\right]\bigg).
	\end{equation}
	If $a_{k+1}(t)+b_{k+1}(t)>0$, from Eqs.~\eqref{eq:sk} and \eqref{eq:s-sab} we can see $d \left(a_{k+1}(t)+b_{k+1}(t)\right)/d t \geq 0$. Thus, $a_{k+1}(t)+b_{k+1}(t)$ is non-decreasing. The case of $a_{k+1}(t)+b_{k+1}(t)<0$ is similar. Therefore, we get $|a_{k+1}(t)+b_{k+1}(t)|$ is non-decreasing.
	Since $\norm{\vu(t)} = \norm{\vv(t)} = 1$, we have $|a_{k+1}(t)+ b_{k+1}(t)| = |\vv_{k+1}^\top\vu(t)+\vv_{k+1}^\top\vv(t)|\leq 2$. Hence, $|a_{k+1}(t)+b_{k+1}(t)|$ converges. 
	Moreover, we note that from 4), $a_{k+1}(t)+ b_{k+1}(t)$ preserves the sign of its initial value, showing that $\lim_{t \to +\infty}a_{k+1}(t)+b_{k+1}(t)$ exists.
\end{proof}

\subsection{Proof of Lemma \ref{lemma:baruv}}
\begin{proof}
	From 1) in Proposition \ref{prop:base} and $s(0) \neq 0$, we obtain $s(t)>0, \forall t>0$.
	
	\textbf{Case 1.} If $\vu(t_0)^\top \mZ \vv(t_0) > 0$ for some $t_0\geq 0$, we get $\mW(t_0) \in \gH_b(\mZ)$ for some $b>0$. Hence, by Lemma \ref{lemma:hb}, we obtain $\bar{s}>0$. 
	From Eq.~\eqref{eq:grad-s}, we obtain
	\[ 0 = \lim_{t \to +\infty}N s(t)^{2-\frac{2}{N}}\left(\vu(t)^\top\mZ\vv(t)-s(t)\right). \]
	Using $s(t) \to \bar{s}>0$ again, we obtain $\lim_{t \to +\infty}\vu(t)^\top\mZ\vv(t)$ exists. Therefore,
	\[  0 = \lim_{t \to +\infty}\frac{d \vu(t)^\top\mZ\vv(t)}{d t} \stackrel{\eqref{eq:mono-uzv}}{=} \lim_{t \to +\infty} s(t)^{1-\frac{2}{N}} \left[\norm{\left(\mI_{d_y}-\vu(t)\vu(t)^\top\right)\mZ\vv(t)}_2^2 +\norm{\left(\mI_{d_x}-\vv(t)\vv(t)^\top\right)\mZ^\top\vu(t)}_2^2 \right].
	\]
	By $s(t) \to \bar{s}>0$, we obtain
	$\left(\mI_{d_y}-\vu(t)\vu(t)^\top\right)\mZ\vv(t) \to \bm{0}$, and $\left(\mI_{d_x}-\vv(t)\vv(t)^\top\right)\mZ^\top\vu(t) \to \bm{0}$.
	Thus, we can choose $t_n=n$ for example. 
	
	\textbf{Case 2.} If $\vu(t)^\top \mZ \vv(t) \leq 0, \forall t \geq 0$. Then from Eq.~\eqref{eq:grad-s}, we get $\dot{s}(t) \leq 0, \forall t\geq 0$. Hence, $s(t) \leq s(0)$. Moreover, by 2) in Proposition \ref{prop:base}, we have $\vu(t)^\top \mZ \vv(t) \geq \vu(0)^\top \mZ \vv(0)$. Therefore,
	\begin{equation}\label{eq:lower-ode}
	\begin{aligned}
	\dot{s}(t) &\stackrel{\eqref{eq:grad-s}}{=} N s(t)^{2-\frac{2}{N}}\left(\vu(t)^\top\mZ\vv(t)-s(t)\right)\ge  N s(t)^{2-\frac{2}{N}}\left(\vu(0)^\top\mZ\vv(0)-s(0)\right).
	\end{aligned}
	\end{equation}
	Now we denote 
	\[ C(a):= \inf_{t\geq a}\norm{\left(\mI_{d_y}-\vu(t)\vu(t)^\top\right)\mZ\vv(t)}_2^2 +\norm{\left(\mI_{d_x}-\vv(t)\vv(t)^\top\right)\mZ^\top\vu(t)}_2^2 \geq 0. \]
	In the following, we show that $C(a)=0, \forall a \geq 0$.
	
	1) If $N=2$, then we can see $\forall t\geq a$, by $\vu(t)^\top \mZ \vv(t) \leq 0, \forall t \geq 0$,
	\begin{align*}
		&- \vu(a)^\top\mZ\vv(a) \geq \vu(t)^\top\mZ\vv(t) - \vu(a)^\top\mZ\vv(a) = \int_{a}^t \frac{d \vu(x)^\top\mZ\vv(x)}{dx} dx \\
		\stackrel{\eqref{eq:mono-uzv}}{=}&
		\int_{a}^t \left[\norm{\left(\mI_{d_y}-\vu(x)\vu(x)^\top\right)\mZ\vv(x)}_2^2 +\norm{\left(\mI_{d_x}-\vv(x)\vv(x)^\top\right)\mZ^\top\vu(x)}_2^2\right] dx \geq C(a)(t-a).
	\end{align*}
	Taking $ t \to +\infty$, we obtain $C(a)=0, \forall a\geq 0$.

	2) If $N>2$, by solving Eq.~\eqref{eq:lower-ode}, we get 
	\[ \frac{N}{2-N} \cdot s(t)^{\frac{2}{N}-1} - \frac{N}{2-N} \cdot s(0)^{\frac{2}{N}-1} \geq N \left(\vu(0)^\top\mZ\vv(0)-s(0)\right)t. \]
	Therefore, we obtain
	\begin{equation}\label{eq:s-upper}
	s(t)^{\frac{2}{N}-1} \leq (N-2)\left(s(0)-\vu(0)^\top\mZ\vv(0)\right)t+s(0)^{\frac{2}{N}-1}:=A+Bt, A,B>0.
	\end{equation}
	Then we can see $\forall t\geq a$,
	\begin{align*}
		- \vu(a)^\top\mZ\vv(a) & \ \geq \vu(t)^\top\mZ\vv(t) - \vu(a)^\top\mZ\vv(a) = \int_{a}^t \frac{d \vu(x)^\top\mZ\vv(x)}{dx} dx \\
		&\stackrel{\eqref{eq:mono-uzv}}{=} \int_{a}^t s(x)^{1-\frac{2}{N}} \left[\norm{\left(\mI_{d_y}-\vu(x)\vu(x)^\top\right)\mZ\vv(x)}_2^2 +\norm{\left(\mI_{d_x}-\vv(x)\vv(x)^\top\right)\mZ^\top\vu(x)}_2^2\right] dx \\
		& \ \geq C(a)\int_{a}^ts(x)^{1-\frac{2}{N}} d x \stackrel{\eqref{eq:s-upper}}{\geq} C(a)\int_{a}^t\frac{1}{A+Bx} d x = \frac{C(a)}{B}\ln\frac{A+B t}{A+Ba}.
	\end{align*}
	Taking $ t \to +\infty$, we obtain $C(a)=0, \forall a\geq 0$.
	
	Therefore, combining 1) and 2), we conclude $C(a)=0, \forall a\geq 0$. Hence, we can find a sequence $\{t_n\}$ with $t_n \to +\infty$, s.t., 
	\begin{align*}
	0 &= \lim_{n \to +\infty}\norm{\left(\mI_{d_y}-\vu(t_n)\vu(t_n)^\top\right)\mZ\vv(t_n)}_2^2 +\norm{\left(\mI_{d_x}-\vv(t_n)\vv(t_n)^\top\right)\mZ^\top\vu(t_n)}_2^2.
	\end{align*}
	Thus, $\lim_{n \to +\infty}\left(\mI_{d_y}-\vu(t_n)\vu(t_n)^\top\right)\mZ\vv(t_n) = \bm{0}$, and $\lim_{n \to +\infty}\left(\mI_{d_x}-\vv(t_n)\vv(t_n)^\top\right)\mZ^\top\vu(t_n) = \bm{0}$. 
	
	Now we adopt the expansion following Eq.~\eqref{eq:uv-exp}: $\vu(t) = \sum_{i=1}^{d_y} a_i(t)\vu_i, \ \vv(t) = \sum_{i=1}^{d_x} b_i(t)\vv_i$. Thus, we have
	\[ \vu(t)^\top\mZ\vv(t) \stackrel{\eqref{eq:uzv}}{=} \sum_{j=1}^d s_j a_j(t)b_j(t), \mZ\vv(t) \stackrel{\eqref{eq:uzv}}{=} \sum_{i=1}^{d} s_i b_i(t) \vu_i, \mZ^\top\vu(t) \stackrel{\eqref{eq:uzv}}{=} \sum_{i=1}^{d} s_i a_i(t) \vv_i. \]
	Therefore, we obtain 
	\begin{equation*}
	\left(\mI_{d_y}-\vu(t_n)\vu(t_n)^\top\right)\mZ\vv(t_n) = \sum_{i=1}^{d_y} \left[s_i b_i(t_n) - \bigg(\sum_{j=1}^d s_j a_j(t_n)b_j(t_n)\bigg) a_i(t_n) \right]\vu_i,
	\end{equation*}
	where we utilize $s_i=0, \forall i> d$. Since $\lim_{n \to +\infty}\left(\mI_{d_y}-\vu(t_n)\vu(t_n)^\top\right)\mZ\vv(t_n) = \bm{0}$, and $\vu_i$s are orthonormal basis, we obtain Eq.~\eqref{eq:abti}.
	Similarly, we could obtain Eq.~\eqref{eq:bati} by $\lim_{n \to +\infty}\left(\mI_{d_x}-\vv(t_n)\vv(t_n)^\top\right)\mZ^\top\vu(t_n) = \bm{0}$. 
	
	Finally, adding the equation in Eq.~\eqref{eq:abti} and Eq.~\eqref{eq:bati}, we obtain
	\[ \lim_{n \to +\infty} \bigg(\sum_{j=1}^d s_j a_j(t_n)b_j(t_n)-s_i\bigg) \left(a_i(t_n)+b_i(t_n)\right) = 0, \forall i \in [d]. \]
	Since we have for some $i_0 \in[d]$ that $\lim_{n \to +\infty} a_{i_0}(t_n)+b_{i_0}(t_n)$ exists and not zero. Thus we obtain
	\[ \lim_{n \to +\infty} \vu(t_n)^\top\mZ\vv(t_n) \stackrel{\eqref{eq:uzv}}{=} \lim_{n \to +\infty} \sum_{j=1}^d s_j a_j(t_n)b_j(t_n) = s_{i_0}. \]
	The proof is finished. 
\end{proof}

\subsection{Proof of Lemma \ref{lemma:getab}}
\begin{proof}
    Since $a_i(t_n)+b_i(t_n)=0, \forall i \in [k], n \geq 0$, we obtain 
    \begin{equation}\label{eq:ab-sign}
        b_i(t_n) = -a_i(t_n), \forall i \in [k], n \geq 0,
    \end{equation}
    and
	\begin{equation*}
		\sum_{j=1}^d s_j a_j(t_n) b_j(t_n) = -\sum_{j=1}^{k} s_j a_j^2(t_n)+\sum_{j=k+1}^d s_ja_j(t_n)b_j(t_n) \stackrel{\eqref{eq:ineq2}}{\leq} s_{k+1} \bigg(1-\sum_{j=1}^k a_j^2(t_n) \bigg).
	\end{equation*}
	Taking limit inferior in both sides and noting that $\lim_{n \to +\infty}\sum_{j=1}^d s_j a_j(t_n)b_j(t_n) = s_{k+1}$, we get
	\begin{equation}\label{eq:limsup}
	    s_{k+1} \leq \liminf_{n \to +\infty} s_{k+1}\left(1-\sum_{j=1}^k a_j^2(t_n)\right).
	\end{equation}
	Moreover, naturally we have 
	\begin{equation}\label{eq:liminf}
	    \limsup_{n \to +\infty} s_{k+1}\left(1-\sum_{j=1}^k a_j^2(t_n)\right) \leq s_{k+1}.
	\end{equation}
	By Eq.~\eqref{eq:limsup} and Eq.~\eqref{eq:liminf}, we obtain $\lim_{n \to +\infty}\sum_{j=1}^k a_j^2(t_n)=0$, showing that
	\begin{equation}\label{eq:abj}
		\lim_{n \to +\infty } -b_j(t_n) \stackrel{\eqref{eq:ab-sign}}{=} \lim_{n \to +\infty }a_j(t_n)=0, \forall j \in [k].
	\end{equation}
	Hence, we derive that
	\begin{equation}\label{eq:sab}
		\lim_{n \to +\infty} \sum_{j=k+1}^d s_j a_j(t_n)b_j(t_n) = \lim_{n \to +\infty} \sum_{j=1}^d s_j a_j(t_n)b_j(t_n) - \lim_{n \to +\infty} \sum_{j=1}^k s_j a_j(t_n)b_j(t_n) \stackrel{\eqref{eq:abj}}{=} s_{k+1}.
	\end{equation}
	Using Cauchy inequality, we have
	\begin{equation}\label{eq:sab-bound}
	    \left[\sum_{j=k+1}^d s_j a_j(t_n)b_j(t_n)\right]^2 \leq \sum_{j=k+1}^d s_j^2a^2_j(t_n) \cdot \sum_{j=k+1}^d b^2_j(t_n) \leq \left(\sum_{j=k+1}^d s_j^2a^2_j(t_n)\right) \cdot \left(1-\sum_{j=1}^k b^2_j(t_n)\right).
	\end{equation}
	Since $\lim_{n \to +\infty} \sum_{j=1}^k b^2_j(t_n)\stackrel{\eqref{eq:abj}}{=} 0$, and $\sum_{j=1}^{d_y} a^2_j(t_n)\stackrel{\eqref{eq:uv-exp}}{=}1$, we obtain
	\[ \lim_{n \to +\infty} \sum_{j=k+1}^{d_y} s_{k+1}^2a^2_j(t_n) \stackrel{\eqref{eq:abj}}{=} s_{k+1}^2 \stackrel{\eqref{eq:sab}}{=} \lim_{n \to +\infty} \left[\sum_{j=k+1}^d s_j a_j(t_n)b_j(t_n)\right]^2 \stackrel{\eqref{eq:sab-bound}}{\leq} \liminf_{n \to +\infty} \sum_{j=k+1}^d s_j^2a^2_j(t_n). \]
	Noting that $s_j = 0, \forall j>d$, we get
	$0 \leq \liminf_{n \to +\infty} \sum_{j=k+2}^{d_y} (s_j^2-s_{k+1}^2)a^2_j(t_n)$. 
	
	However, $s_j^2-s_{k+1}^2 < 0, \forall j\geq k+2$ and $a^2_j(t_n) \geq 0$, showing that
	$\limsup_{n \to +\infty} \sum_{j=k+2}^{d_y} (s_j^2-s_{k+1}^2)a^2_j(t_n) \leq 0$. 
	Hence, we obtain
	$\lim_{n \to +\infty} a_j(t_n) =0, \forall j \geq k+2$.
	The similar analysis holds for $b_{j}(t_n)$. Therefore, we obtain 
	\[ \lim_{n \to +\infty} s_{k+1}a_{k+1}(t_n)b_{k+1}(t_n) = \lim_{n \to +\infty} \sum_{j=1}^d s_j a_j(t_n)b_j(t_n) =s_{k+1}. \]
	Finally, we have 
	\[ \vu_{k+1}\vv_{k+1}^\top \stackrel{\eqref{eq:ab-limit}}{=} \lim_{n \to +\infty}a_{k+1}(t_n)b_{k+1}(t_n)\vu_{k+1}\vv_{k+1}^\top \stackrel{\eqref{eq:ab-limit}}{=} \lim_{n \to +\infty}\sum_{i,j}a_i(t_n)b_j(t_n)\vu_i\vv_j^\top \stackrel{\eqref{eq:uv-exp}}{=} \lim_{n \to +\infty}\vu(t_n)\vv(t_n)^\top. \]
	The proof is finished.
\end{proof}

\subsection{Proof of Theorem \ref{thm:monotone}}
\begin{proof}
    (I) The truth that $\dot{c}_1(t)\ge 0$ is direct from 3) in Proposition \ref{prop:base}. Moreover, from Theorem \ref{thm:ab} with $k=0$, we have $c_1(t)=a_1(t)b_1(t) \to 1$. Thus, we obtain $t_1<+\infty$.
    
    (II) As for $s(t)$, if $t_2=+\infty$, then $\vu(t)^\top\mZ\vv(t) \le s(t), \forall t\in[0,+\infty)$. Thus by Eq.~\eqref{eq:grad-s}$, \dot{s}(t) \le 0$ for all $t\ge 0$.
    Now we consider the remaining case where $t_2<+\infty$. Since $t_2=\inf\{t: \vu(t)^\top\mZ\vv(t) \ge s(t)\}$, we have $\vu(t)^\top\mZ\vv(t) \le s(t)$ when $t\in[0,t_2)$. Thus
    \begin{equation*}
        \dot{s}(t) \stackrel{\eqref{eq:grad-s}}{=} N s(t)^{2-\frac{2}{N}}\left(\vu(t)^\top\mZ\vv(t)-s(t)\right) \leq 0, \forall t\in[0, t_2).
    \end{equation*}
    Now from $t_2< +\infty$, we get $\vu(t_2)^\top\mZ\vv(t_2) = s(t_2)$.
    We denote $T = \inf\{t: \dot{\vu}(t) = \bm{0}, \dot{\vv}(t) = \bm{0}\}$. 
    
    (i) When $T>t_2$. Then for $t\in [t_2, T)$, we have $\dot{\vu}(t) \neq \bm{0}$ or $\dot{\vv}(t) \neq \bm{0}$. 
    Thus, applying lemma \ref{lemma:s(t)-two}, we have 
    \[ \vu(t)^\top\mZ\vv(t) = s(t) \Rightarrow d(\vu(t)^\top\mZ\vv(t)-s(t)) / d t>0, \forall t \in [t_2, T). \]
    By \citet[Lemma 10]{lin2021faster}, we obtain $\vu(t)^\top\mZ\vv(t) \geq s(t), \forall t \in[t_2, T)$. Hence, 
    \begin{equation*}
        \dot{s}(t) = N s(t)^{2-\frac{2}{N}}\left(\vu(t)^\top\mZ\vv(t)-s(t)\right) \geq 0, \forall t\in[t_2, T).
    \end{equation*}    
    And for $t\geq T$, we get $T < +\infty$. We obtain stationary singular vectors from time $T$ by Proposition \ref{prop:station-vector}. 
    Thus, $\vu(T)^\top\mZ\vv(T)=c ,\forall t \geq T$ for a constant $c$, which reduce the variation of $s(t)$ as 
    \[ \dot{s}(t) = N s(t)^{2-\frac{2}{N}}\left(c-s(t)\right), \forall t \geq T. \]
    Moreover, since $\vu(t)^\top\mZ\vv(t) \geq s(t), \forall t \in[t_2, T)$, we obtain $c = \vu(T)^\top\mZ\vv(T) \geq s(T)$.
    Hence, we can see $\dot{s}(t) \geq 0, \forall t \geq T$. 
    
    (ii) When $T \leq t_2$, we have $T < +\infty$. We obtain stationary singular vectors from time $T$ by Proposition \ref{prop:station-vector}. 
    Thus, $\vu(T)^\top\mZ\vv(T)=c ,\forall t \geq T$ for a constant $c$, which reduce the variation of $s(t)$ as 
    \[ \dot{s}(t) = N s(t)^{2-\frac{2}{N}}\left(c-s(t)\right), \forall t \geq T. \]
    We note that $c = \vu(t_2)^\top\mZ\vv(t_2) = s(t_2)$. Thus $\dot{s}(t) =0, \forall t\geq t_2$.
    The proof is finished.
\end{proof}

\subsection{Proof of Theorem \ref{thm:s}}
\begin{proof}
We first consider the upper bound. From $s(0)>0$ and 1) in Proposition \ref{prop:base}, we have $s(t)>0$ for all $t\ge 0$. 
Moreover, we have 
\begin{equation}\label{eq:upper-s1}
\dot{s}(t) \stackrel{\eqref{eq:grad-s}}{=} N s(t)^{2-\frac{2}{N}} \left(\sum_{j=1}^d s_j a_j(t)b_j(t)-s(t)\right) \stackrel{\eqref{eq:ineq2}}{\leq} N s(t)^{2-\frac{2}{N}} \left(s_1 -s(t)\right).
\end{equation}
Let $\tilde{s}(t)$ be the solution of the ODE
\[ \dot{\tilde{s}}(t) = N \tilde{s}(t)^{2-\frac{2}{N}} \left(s_1 -\tilde{s}(t)\right), \quad \tilde{s}(0)=s(0). \]
Then we can see $s(t) \leq \tilde{s}(t)$ from Eq.~\eqref{eq:upper-s1}.

If $s(0)>s_1$, then $\tilde{s}(t) \leq \tilde{s}(0)=s(0)$, showing that $s(t) \leq \tilde{s}(t) \leq s(0)$.
Otherwise, $s(0) \leq s_1$, we get $\tilde{s}(t) \leq s_1$, showing that $s(t) \leq \tilde{s}(t) \leq s_1$.
Therefore, we know $s(t)< s_0:=\max\{s_1,s(0)\}$ for all $t \geq 0$.

Now we consider the lower bound. Note that
\begin{equation}\label{eq:pre-sab}
\begin{aligned}
    \sum_{j=2}^d s_j a_j(t)b_j(t) & \stackrel{\eqref{eq:ineq2}}{\leq} s_2 \sqrt{(1-a_1^2(t))(1-b_1^2(t))} = s_2 \sqrt{a_1^2(t)b_1^2(t)-a_1^2(t)-b_1^2(t)+1} \\
    &\ \leq \ s_2 \sqrt{a_1^2(t)b_1^2(t)-2|a_1(t)b_1(t)|+1} =s_2\left(1-|a_1(t)b_1(t)|\right),
\end{aligned}
\end{equation}
where we use $|a_1(t)b_1(t)| = |\vu_1^\top\vu(t) \cdot \vv_1^\top\vv(t)| \leq 1$ in the last equality.

\begin{equation}\label{eq:ode-s-lower}
\begin{aligned}
    \dot{s}(t) &\stackrel{\eqref{eq:grad-s}}{=} N s(t)^{2-\frac{2}{N}} \bigg(\sum_{j=1}^d s_j a_j(t)b_j(t)-s(t)\bigg) \stackrel{\eqref{eq:pre-sab}}{\geq} N s(t)^{2-\frac{2}{N}} \left[ s_1 a_1(t)b_1(t) - s_2 \left(1-|a_1(t)b_1(t)|\right) - s(t)\right] \\
    & \geq N s(t)^{2-\frac{2}{N}}\bigg[(s_2-s_1) \left|a_1(t)b_1(t)\right|-s_2-s(t)\bigg] \geq -N (s_1+s(t))s(t)^{2-\frac{2}{N}} \geq -N (s_1+s_0)s(t)^{2-\frac{2}{N}},
\end{aligned}
\end{equation}
where the last inequality uses $s(t)< s_0:=\max\{s_1,s(0)\}$, which is proved previously.

When $N=2$, we solve Eq.~\eqref{eq:ode-s-lower} and get
\[ \dot{s}(t) \stackrel{\eqref{eq:ode-s-lower}}{\geq} -2(s_1+s_0)s(t) \Rightarrow  \frac{d \left(\ln s(t)\right)}{dt} \geq -2(s_1+s_0) \Rightarrow  \ln \frac{s(t)}{s(0)} \geq -2(s_1+s_0)t \Rightarrow s(t) \geq s(0)e^{-2(s_1+s_0)t}. \]
When $N \ge 3$, we solve Eq.~\eqref{eq:ode-s-lower} and get
\[ \frac{N}{2-N} \cdot \frac{d \left(s(t)^{\frac{2}{N}-1}\right)}{dt} \stackrel{\eqref{eq:ode-s-lower}}{\geq}-N (s_1+s_0) \Rightarrow  s(t)^{\frac{2}{N}-1} - s(0)^{\frac{2}{N}-1} \leq (s_1+s_0)(N-2)t. \]
Thus, we finally obtain
\[ s(t) \geq \left[(s_1+s_0)(N-2)t+s(0)^{\frac{2}{N}-1}\right]^{-\frac{N}{N-2}}. \]
The proof of Eqs.~\eqref{eq:s-upper-2} and \eqref{eq:s-upper-3} is finished. 
\end{proof}

\subsection{Proof of Theorem \ref{thm:stage1}}
\begin{proof}
    Since $a_1(0)b_1(0)<0, a_1(0)+b_1(0) \neq 0$, without loss of generality, we suppose $a_1(0)>0, b_1(0)<0$ and $a_1(0)+b_1(0)>0$. Note that 
    \begin{equation*}
        \dot{a}_1(t)-\dot{b}_1(t) \stackrel{\eqref{eq:grad-ab}}{=} s(t)^{1-\frac{2}{N}}\left(b_1(t)-a_1(t)\right) \bigg(s_1+\sum_{j=1}^d\left[s_j a_j(t)b_j(t)\right]\bigg).
    \end{equation*}
    By \citet[Lemma 4]{arora2019implicit} and $|a_1(t)-b_1(t)| \leq 2$, we get that $a_1(t)-b_1(t)$ preserves the sign of its initial value:
    \begin{equation}\label{eq:a-ge-b}
        a_1(t)-b_1(t)>0, \forall t\geq 0.
    \end{equation} 
    Moreover, from 5) in Proposition \ref{prop:base} and $a_1(0)+b_1(0)>0$, we obtain 
    \begin{equation}\label{eq:a+b}
        a_1(t)+b_1(t) \geq a_1(0)+b_1(0)>0, \forall t \geq 0.
    \end{equation}
    Then we have
    \begin{equation}\label{eq:a-pos}
        a_1(t) \stackrel{\eqref{eq:a-ge-b}}{\geq} \frac{a_1(t)+b_1(t)}{2} \stackrel{\eqref{eq:a+b}}{\geq} \frac{a_1(0)+b_1(0)}{2}>0, \forall t \geq 0,
    \end{equation}
    and 
    \begin{equation}\label{eq:abs b/a}
        -a_1(t) \stackrel{\eqref{eq:a+b}}{\leq} b_1(t)\stackrel{\eqref{eq:a-ge-b}}{\leq} a_1(t), \forall t \geq 0 \stackrel{\eqref{eq:a-pos}}{\Rightarrow} -1 < \frac{b_1(t)}{a_1(t)} < 1,  \forall t \geq 0.
    \end{equation}
    Furthermore, we can derive that
    \begin{equation*}
        \frac{d}{dt}\left(\frac{b_1(t)}{a_1(t)}\right)=\frac{\dot{b}_1(t)a_1(t)-\dot{a}_1(t)b_1(t)}{a_1^2(t)} \stackrel{\eqref{eq:grad-ab}}{=} s(t)^{1-\frac{2}{N}}\left(1-\left(\frac{b_1(t)}{a_1(t)}\right)^2\right) \stackrel{\eqref{eq:abs b/a}}{>} 0.
    \end{equation*}
    Solving the above ODE, we obtain
    \begin{equation*}
        d \left(\ln \sqrt{\frac{a_1(t)+b_1(t)}{a_1(t)-b_1(t)}} \right) / dt \geq s(t)^{1-\frac{2}{N}} \stackrel{\eqref{eq:s-upper-3}}{\geq} \frac{1}{2(s_1+s_0)(N-2)t+s(0)^{\frac{2}{N}-1}}.
    \end{equation*}
    Therefore, we obtain
    \begin{equation*}
    \begin{aligned}
        & \ln \sqrt{\frac{a_1(t)+b_1(t)}{a_1(t)-b_1(t)}}-\ln \sqrt{\frac{a_1(0)+b_1(0)}{a_1(0)-b_1(0)}} \geq \int_{0}^t \frac{dx}{2(s_1+s_0)(N-2)x+s(0)^{\frac{2}{N}-1}} \\
        =& \ \frac{1}{2(s_1+s_0)(N-2)}\ln \left[1+\frac{2(s_1+s_0)(N-2)t}{s(0)^{\frac{2}{N}-1}}\right] = \frac{1}{2(s_1+s_0)(N-2)}\ln \left[1+\frac{2(s_1+s_0)(N-2)t}{s(0)^{\frac{2}{N}-1}}\right].
    \end{aligned}
    \end{equation*}
    We hide constants related to initialization, and rewrite the inequality as 
    \begin{equation}\label{eq:ab-lower}
        \frac{a_1(t)+b_1(t)}{a_1(t)-b_1(t)} \geq C_1(1+A_1t)^{B_1},
    \end{equation}
    where $A_1 := \frac{2(s_1+s_0)(N-2)}{s(0)^{\frac{2}{N}-1}}$, $B_1 := \frac{1}{(s_1+s_0)(N-2)}$, $1>C_1 := \frac{a_1(0)+b_1(0)}{a_1(0)-b_1(0)} \stackrel{\eqref{eq:abs b/a}}{>} 0$.
    Hence, we obtain
    \begin{equation}\label{eq:ab-bound}
        a_1(t)b_1(t) \stackrel{\eqref{eq:ab-lower}, \eqref{eq:a-pos}}{\geq} \frac{C_1(1+A_1t)^{B_1}-1}{C_1(1+A_1t)^{B_1}+1} \cdot a_1^2(t).
    \end{equation}
    Then we can see $a_1(t)b_1(t) \geq 0 $ provided $C_1(1+A_1t)^{B_1} > 1$, i.e., 
    \[ t \geq T_1 = \frac{C_1^{-1/B_1}-1}{A_1} = \frac{s(0)^{\frac{2}{N}-1}}{2(s_1+s_0)(N-2)} \cdot \left[ \left(\frac{a_1(0)-b_1(0)}{a_1(0)+b_1(0)}\right)^{(s_1+s_0)(N-2)} - 1 \right]. \]
    Therefore, we obtain $t_1 \leq T_1$.
    Moreover, when $t \leq T_1$, by $a_1^2(t) \leq 1$, we have
    \[ a_1(t)b_1(t) \stackrel{\eqref{eq:ab-bound}}{\geq} \frac{C_1(1+A_1t)^{B_1}-1}{C_1(1+A_1t)^{B_1}+1}.  \]
    That is, $1-a_1(t)b_1(t) = \gO((1+A_1t)^{-B_1}) = \gO([(N-2)t]^{-\frac{c_1}{N-2}})$.
    
    Additionally, when $A_1 t \geq e \cdot C_1^{-1/B_1}-1$, we have
    \begin{equation}\label{eq:tt}
        A_1 t \geq e \cdot C_1^{-1/B_1}-1 \geq \left(\frac{1+B_1}{C_1}\right)^{1/B_1}-1 \Rightarrow C_1(1+A_1t)^{B_1} \geq 1+B_1.
    \end{equation}
    Thus, we get
    \[ a_1(t)b_1(t) \stackrel{\eqref{eq:ab-bound}}{\geq} \frac{C_1(1+A_1t)^{B_1}-1}{C_1(1+A_1t)^{B_1}+1} \cdot a_1^2(t) \stackrel{\eqref{eq:a-pos}, \eqref{eq:tt}}{\geq} \frac{B_1}{2+B_1} \cdot \frac{(a_1(0)+b_1(0))^2}{4}=\Theta\left(\frac{(a_1(0)+b_1(0))^2}{N}\right)>0.  \]
\end{proof}

\subsection{Proof of Theorem \ref{thm:stage2}}
\begin{proof}
    Since $a_1(0)b_1(0)\ge 0$, then by 3) in Proposition \ref{prop:base}, we know $a_1(t)b_1(t)\ge 0$ for all $t\ge0$. Now we consider the flow of $a_1(t)b_1(t)$.
    \begin{equation}\label{eq:lower-grad ab}
    	\begin{aligned}
    	\frac{d a_1(t)b_1(t)}{d t} & \stackrel{\eqref{eq:grad-ab}}{=} s(t)^{1-\frac{2}{N}}\bigg(s_1b_1(t)^2 + s_1a_1(t)^2-2a_1(t)b_1(t) \sum_{j=1}^d\left[s_j a_j(t)b_j(t)\right]\bigg)\\
    	&\stackrel{\eqref{eq:pre-sab}}{\geq} s(t)^{1-\frac{2}{N}}\bigg[2s_1a_1(t)b_1(t)-2a_1(t)b_1(t)\left( s_1a_1(t)b_1(t)+s_2\left(1-|a_1(t)b_1(t)|\right) \right)\bigg] \\
    	&\ = \ 2a_1(t)b_1(t)s(t)^{1-\frac{2}{N}}\bigg[s_1-s_1a_1(t)b_1(t)-s_2\left(1-a_1(t)b_1(t)\right)\bigg] \\
    	&\ = \ 2a_1(t)b_1(t)s(t)^{1-\frac{2}{N}}\left(s_1-s_2\right)\left(1-a_1(t)b_1(t)\right).
        \end{aligned}
    \end{equation}
By the lower bound of $s(t)$ in Theorem \ref{thm:s}, we obtain
\[ \frac{d a_1(t)b_1(t)}{d t} \stackrel{\eqref{eq:lower-grad ab}}{\geq} 2a_1(t)b_1(t)s(t)^{1-\frac{2}{N}}\left(s_1-s_2\right)\left(1-a_1(t)b_1(t)\right) \stackrel{\eqref{eq:s-upper-3}}{\geq} \frac{2a_1(t)b_1(t) \left(s_1-s_2\right)\left(1-a_1(t)b_1(t)\right)}{(s_0+s_1)(N-2)t+s(0)^{\frac{2}{N}-1}}.  \]
Denoting $c_1(t) := a_1(t)b_1(t)$, we get
\[ \ln \frac{c_1(t)}{1-c_1(t)} -\ln \frac{c_1(0)}{1-c_1(0)} \geq \frac{2(s_1-s_2)}{(s_1+s_0)(N-2)}\ln \left(1+\frac{(s_1+s_0)(N-2)t}{s(0)^{\frac{2}{N}-1}}\right). \]
Further we can rewrite the bound as 
\begin{equation*}
    c_1(t) \geq 1-\frac{1}{A(1+B(N-2)t)^{\frac{c_2}{N-2}}+1},
\end{equation*}
where $A=\frac{c_1(0)}{1-c_1(0)}>0, B=(s_1+s_0)s(0)^{1-\frac{2}{N}}>0, c_2=\frac{2(s_1-s_2)}{s_1+s_0}>0$.
Then we have $1-a_1(t)b_1(t) = \gO\left([(N-2)t]^{-c_2/(N-2)}\right)$. The proof is finished.
\end{proof}

\subsection{Proof of Theorem \ref{thm:stage2.5}}
\begin{proof}
    When $t_2 = +\infty$, we have $\vu(t)^\top\mZ\vv(t) < s(t), \forall t \geq 0$. Thus, we obtain 
    \begin{equation}\label{eq:dot-s}
        \dot{s}(t) \stackrel{\eqref{eq:grad-s}}{=} N s(t)^{2-\frac{2}{N}} \left(\vu(t)^\top\mZ\vv(t)-s(t)\right)\leq 0.
    \end{equation}
    We note that by Theorem \ref{thm:ab}, $s(t) \to s_1$. Thus, we conclude
    \begin{equation}\label{eq:lower-s-2}
        s(t) \stackrel{\eqref{eq:dot-s}}{\geq} \lim_{t \to +\infty} s(t) = s_1.
    \end{equation}
    Then we have
    \begin{equation*}
    \begin{aligned}
        \frac{d a_1(t)b_1(t)}{d t} &\stackrel{\eqref{eq:lower-grad ab}}{\geq} 2a_1(t)b_1(t)s(t)^{1-\frac{2}{N}}\left(s_1-s_2\right)\left(1-a_1(t)b_1(t)\right) \\
        & \stackrel{\eqref{eq:lower-s-2}}{\geq} 2a_1(t)b_1(t)s_1^{1-\frac{2}{N}} \left(s_1-s_2\right)\left(1-a_1(t)b_1(t)\right).
    \end{aligned}
    \end{equation*}
    Setting $c_1(t):=a_1(t)b_1(t)$ and solving the ODE in the above, we get
    \[ \ln \frac{c_1(t)}{1-c_1(t)} -\ln \frac{c_1(0)}{1-c_1(0)} \geq 2 s_1^{1-\frac{2}{N}}\left(s_1-s_2\right)t. \]
    We rewrite the bound to 
    \[ 1-c_1(t)\le \left(1+\frac{c_1(0)}{1-c_1(0)} \cdot e^{2s_1^{1-\frac{2}{N}}\left(s_1-s_2\right)t}\right)^{-1}, i.e., 1-c_1(t)= \gO\left(e^{-c_3t}\right). \]
    
    To obtain the bound of $s(t)$, we notice that
    \begin{equation*}
        \dot{s}(t) \stackrel{\eqref{eq:grad-s}}{=} N s(t)^{2-\frac{2}{N}} \left(\sum_{j=1}^d s_j a_j(t)b_j(t)-s(t)\right) \stackrel{\eqref{eq:ineq2}}{\leq} N s(t)^{2-\frac{2}{N}} \left(s_1 -s(t)\right) \leq N s_1^{2-\frac{2}{N}} \left(s_1 -s(t)\right).
    \end{equation*}
    We can obtain the upper bound of the evolution $s(t)$ as
    \begin{equation*}
        d \left(\ln \left(s(t)-s_1\right) \right) / dt \leq -N s_1^{2-\frac{2}{N}} \Rightarrow s(t) \leq s_1 +(s(0)-s_1)e^{-N s_1^{2-\frac{2}{N}}t}, i.e., s(t)-s_1=\gO\left(e^{-c_4 t}\right).
    \end{equation*}
    Finally, noting that $s(t) \stackrel{\eqref{eq:lower-s-2}}{\geq} s_1$, we obtain $|s(t)-s_1| = \gO\left(e^{-c_4 t}\right)$. The proof is finished.
\end{proof}

\subsection{Proof of Theorem \ref{thm:stage3}}
\begin{proof}
    By Theorem \ref{thm:monotone}, we have $\dot{s}(t) \geq 0, \forall t\geq 0$. Thus, we can lower bound $s(t) \geq s(0)>0$. Then we obtain 
    \begin{equation*}
        \frac{d a_1(t)b_1(t)}{d t} \stackrel{\eqref{eq:lower-grad ab}}{\geq} 2a_1(t)b_1(t)s(t)^{1-\frac{2}{N}}\left(s_1-s_2\right)\left(1-a_1(t)b_1(t)\right) \geq 2a_1(t)b_1(t)s(0)^{1-\frac{2}{N}} \left(s_1-s_2\right)\left(1-a_1(t)b_1(t)\right).
    \end{equation*}
    Setting $c_1(t):=a_1(t)b_1(t)$ and solving the ODE in the above, we get
    \begin{equation}\label{eq:c1-lower}
        1-c_1(t)\le \left(1+\frac{c_1(0)}{1-c_1(0)} \cdot e^{2s(0)^{1-\frac{2}{N}}\left(s_1-s_2\right)t}\right)^{-1}, i.e., 1-c_1(t)= \gO\left(e^{-c_5t}\right).
    \end{equation}
    
    Next we derive the bound for $s(t)$. We continue from
    \begin{equation*}
        \begin{aligned}
        \dot{s}(t) &\stackrel{\eqref{eq:grad-s}}{=} N s(t)^{2-\frac{2}{N}} \bigg(\sum_{j=1}^d s_j a_j(t)b_j(t)-s(t)\bigg) \stackrel{\eqref{eq:pre-sab}}{\geq} N s(t)^{2-\frac{2}{N}} \left[ s_1 a_1(t)b_1(t) - s_2 \left(1-|a_1(t)b_1(t)|\right) - s(t)\right] \\
        &\geq  N s(t)^{2-\frac{2}{N}} \left[ (s_1 + s_2) a_1(t)b_1(t)-s_2-s(t)\right] \stackrel{\eqref{eq:c1-lower}}{\geq} N s(0)^{2-\frac{2}{N}} \left[ -(s_1 + s_2)\left(1+A e^{c_5t}\right)^{-1} + s_1-s(t)\right].
        \end{aligned}
    \end{equation*}
    Solving the above ODE, we arrive at
    \begin{equation*}
         \frac{d(s(t)e^{c_6 t})}{dt} \geq c_6 e^{c_6 t}\left[s_1 -(s_1 + s_2)\left(1+A e^{c_5t}\right)^{-1} \right].
    \end{equation*}
    Hence, we get
    \begin{align*}
        & \ s(t)e^{c_6 t}-s(0) \geq s_1\left(e^{c_6 t}-1\right) - \int_{0}^t c_6e^{c_6x}(s_1 + s_2)\left(1+A e^{c_5x}\right)^{-1} dx \\ 
        \geq & \ s_1\left(e^{c_6 t}-1\right) - e^{c_6t}\int_{0}^t c_6(s_1 + s_2)\left(1+A e^{c_5x}\right)^{-1} dx = s_1\left(e^{c_6 t}-1\right) - \frac{(s_1 + s_2)c_6e^{c_6t} }{c_5} \cdot \ln \left(1+A^{-1} e^{-c_5t}\right) \\
        \geq & \ s_1\left(e^{c_6 t}-1\right) - \frac{(s_1 + s_2)c_6e^{c_6t} }{Ac_5} \cdot e^{-c_5t}. \\
    \end{align*}
    Therefore, we obtain
    \[ s(t) \geq s_1- (s(0)+s_1) e^{-c_6 t} - \frac{(s_1 + s_2)c_6}{Ac_5} \cdot e^{-c_5t}. \]
    After hiding the constants and noting that $s(t)$ is non-decreasing, we obtain
    \begin{equation*}
         s_1 - s(t) =\gO\left(e^{-\min\{c_5, c_6\}t} \right).
    \end{equation*}
     We note that by Theorem \ref{thm:ab}, $s(t) \to s_1$. Since $s(t)$ is non-decreasing, we conclude $s(t) \le \lim_{t \to +\infty} s(t) = s_1$.
    Then we obtain 
    \begin{equation*}
    |s_1 - s(t)| =\gO\left(e^{-\min\{c_5, c_6\}t} \right) .   
    \end{equation*}
    The proof is finished.
\end{proof}

\section{Convergence Rates: $N=2$}\label{case:N=2}
We provide convergence rates of the case $N=2$ in this section. Corresponding to the case $N\ge 3$, we list the rates of three stages as below:

\paragraph{Stage 1.}
For $t \in [0, t_1]$, where $t_1 := \inf \{t: a_1(t)b_1(t) \ge0 \} < +\infty$, we have $a_1(t)b_1(t) \leq 0$, and the rates are
\[ 1-a_1(t)b_1(t) = \gO(e^{-2t}), \quad s(t)=\Omega\left(e^{-2(s_1+s_0)t}\right). \]

\paragraph{Stage 2.}
For $t \in (t_1, t_2]$, where $t_2:= \inf \{t: \vu(t)^\top\mZ\vv(t) \ge s(t)\}$, we have $a_1(t)b_1(t) > 0, \dot{s}(t)\le 0$, and
\[1-a_1(t)b_1(t) = \gO\left(e^{-2(s_1-s_2)t} \right), \quad s(t)=\Omega\left(e^{-2(s_1+s_0)t}\right). \]

\paragraph{Stage 3.}
For $t \in (\max\{t_1, t_2\}, +\infty)$, we have $a_1(t)b_1(t) > 0$ and $\dot{s}(t) \geq 0$, and the rates are
\[ 1-a_1(t)b_1(t) = \gO\left(e^{-2(s_1-s_2)t} \right), \quad |s_1-s(t)|=\gO\left(e^{-\min\{2\left(s_1-s_2\right), 2s(0)\}t} \right). \]

We have shown that the definitions of the stages are well defined by Theorem \ref{thm:monotone}. The convergence rates of $s(t)$, i.e. $s(t)=\Omega\left(e^{-2(s_1+s_0)t}\right)$ in Stage 1 and Stage 2 are given by Theorem \ref{thm:s}.

\subsection*{Convergence Rates of $a_1(t)b_1(t)$: Stage 1}

\begin{thm}\label{thm:ab<0,N=2}
	Suppose $N=2$, $a_1(0)b_1(0)< 0$ and $a_1(0) + b_1(0) \neq 0$. Then we have 
	\[ 1-a_1(t)b_1(t) = \gO(e^{-2t}),  0 \leq t \leq t_1. \]
	
	Furthermore, we have the upper bound of $t_1$ below:
	\begin{equation}\label{eq:t1-upper-1}
	    t_1 \leq \frac{1}{2} \ln \left| \frac{a_1(0)-b_1(0)}{a_1(0)+b_1(0)} \right|.
	\end{equation}
	Additionally, we could obtain
	\begin{equation}\label{eq:ab-lower-33}
	    a_1(t)b_1(t) = \Omega\left((a_1(0)+b_1(0))^2\right), \text{ if } t \geq \frac{1}{2} \ln \left| \frac{2(a_1(0)-b_1(0))}{a_1(0)+b_1(0)} \right|.
	\end{equation}
\end{thm}

\begin{proof}
    Since $a_1(0)b_1(0)<0, a_1(0)+b_1(0) \neq 0$, without loss of generality, we suppose $a_1(0)>0, b_1(0)<0$ and $a_1(0)+b_1(0)>0$. Note that 
    \begin{equation*}
        \dot{a}_1(t)-\dot{b}_1(t) \stackrel{\eqref{eq:grad-ab}}{=} s(t)^{1-\frac{2}{N}}\left(b_1(t)-a_1(t)\right) \bigg(s_1+\sum_{j=1}^d\left[s_j a_j(t)b_j(t)\right]\bigg)=\left(b_1(t)-a_1(t)\right) \bigg(s_1+\sum_{j=1}^d\left[s_j a_j(t)b_j(t)\right]\bigg).
    \end{equation*}
    By \citet[Lemma 4]{arora2019implicit} and $|a_1(t)-b_1(t)| \leq 2$, we get that $a_1(t)-b_1(t)$ preserves the sign of its initial value:
    \begin{equation}\label{eq:a-ge-b-2}
        a_1(t)-b_1(t)>0, \forall t\geq 0.
    \end{equation} 
    Moreover, from 5) in Proposition \ref{prop:base} and $a_1(0)+b_1(0)>0$, we obtain 
    \begin{equation}\label{eq:a+b-2}
        a_1(t)+b_1(t) \geq a_1(0)+b_1(0)>0, \forall t \geq 0.
    \end{equation}
    Then we have
    \begin{equation}\label{eq:a-pos-2}
        a_1(t) \stackrel{\eqref{eq:a-ge-b-2}}{\geq} \frac{a_1(t)+b_1(t)}{2} \stackrel{\eqref{eq:a+b-2}}{\geq} \frac{a_1(0)+b_1(0)}{2}>0, \forall t \geq 0,
    \end{equation}
    and 
    \begin{equation}\label{eq:abs b/a-2}
        -a_1(t) \stackrel{\eqref{eq:a+b-2}}{<} b_1(t)\stackrel{\eqref{eq:a-ge-b-2}}{<} a_1(t), \forall t \geq 0 \stackrel{\eqref{eq:a-pos-2}}{\Rightarrow} -1 < \frac{b_1(t)}{a_1(t)} < 1,  \forall t \geq 0.
    \end{equation}
    Furthermore, we can derive that
    \begin{equation*}
        \frac{d}{dt}\left(\frac{b_1(t)}{a_1(t)}\right)=\frac{\dot{b}_1(t)a_1(t)-\dot{a}_1(t)b_1(t)}{a_1^2(t)} \stackrel{\eqref{eq:grad-ab}}{=} \left(1-\left(\frac{b_1(t)}{a_1(t)}\right)^2\right) \stackrel{\eqref{eq:abs b/a-2}}{>} 0.
    \end{equation*}
    Then we have
    \begin{equation}\label{eq:b/a-1}
        d \left(\ln \sqrt{\frac{a_1(t)+b_1(t)}{a_1(t)-b_1(t)}} \right) / d t = 1 \Rightarrow \frac{b_1(t)}{a_1(t)}=\frac{e^{2t}\left[\frac{a_1(0)+b_1(0)}{a_1(0)-b_1(0)}\right]-1}{e^{2t}\left[\frac{a_1(0)+b_1(0)}{a_1(0)-b_1(0)}\right]+1}.
    \end{equation}
    Thus, we get
    \begin{equation}\label{eq:ab-bound2}
        a_1(t)b_1(t) = a_1^2(t) \cdot \frac{b_1(t)}{a_1(t)}  \stackrel{\eqref{eq:b/a-1}}{=}\frac{A_2 e^{2t}-1}{A_2 e^{2t}+1} \cdot a_1^2(t), A_2 := \frac{a_1(0)+b_1(0)}{a_1(0)-b_1(0)}.
    \end{equation}
    Then we can see $a_1(t)b_1(t) \geq 0 $ provided $A_2 e^{2t} \geq 1$, i.e., 
    \[ t \geq T_2 := \frac{1}{2} \ln \frac{a_1(0)-b_1(0)}{a_1(0)+b_1(0)}. \]
    Therefore, Eq.~\eqref{eq:t1-upper-1} is proved.
    Moreover, when $t \leq T_2$, by $a_1^2(t) \leq 1$, we have
    \[ a_1(t)b_1(t) \stackrel{\eqref{eq:ab-bound2}}{\geq}\frac{A_2 e^{2t}-1}{A_2 e^{2t}+1}.  \]
    That is, $1-a_1(t)b_1(t) = \gO(e^{-2t})$.
    
    Additionally, when $t \geq \frac{1}{2} \ln \frac{2(a_1(0)-b_1(0))}{a_1(0)+b_1(0)}$, we have $A_2 e^{2t} \geq 2$. Hence, we derive that
    \[ a_1(t)b_1(t) \stackrel{\eqref{eq:ab-bound2}}{=}\frac{A_2 e^{2t}-1}{A_2 e^{2t}+1} \cdot a_1^2(t) \stackrel{\eqref{eq:a-pos-2}}{\geq} \frac{(a_1(0)+b_1(0))^2}{12}=\Theta\left((a_1(0)+b_1(0))^2\right)>0. \]
    Thus, Eq.~\eqref{eq:ab-lower-33} is proved.
\end{proof}

\subsection*{Convergence Rates of $a_1(t)b_1(t)$: Stage 2 and Stage 3}
\begin{thm}
	Assume $N=2$, $a_1(0)b_1(0) > 0$. Then we have
	\begin{equation*}
	    1-a_1(t)b_1(t) = \gO\left(e^{-2(s_1-s_2)t} \right).
	\end{equation*}
\end{thm}
\begin{proof}
    Since $a_1(0)b_1(0)> 0$, then by 3) in Proposition \ref{prop:base}, we know $a_1(t)b_1(t)> 0$ for all $t\ge0$. Now we consider the flow of $a_1(t)b_1(t)$.
    \begin{equation}\label{eq:lower-grad ab-2}
    	\begin{aligned}
    	\frac{d a_1(t)b_1(t)}{d t} & \stackrel{\eqref{eq:grad-ab}}{=} s(t)^{1-\frac{2}{N}}\bigg(s_1b_1(t)^2 + s_1a_1(t)^2-2a_1(t)b_1(t) \sum_{j=1}^d\left[s_j a_j(t)b_j(t)\right]\bigg)\\
    	&\stackrel{\eqref{eq:pre-sab}}{\geq} s(t)^{1-\frac{2}{N}}\bigg[2s_1a_1(t)b_1(t)-2a_1(t)b_1(t)\left( s_1a_1(t)b_1(t)+s_2\left(1-|a_1(t)b_1(t)|\right) \right)\bigg] \\
    	&\ = \ 2a_1(t)b_1(t)\bigg[s_1-s_1a_1(t)b_1(t)-s_2\left(1-a_1(t)b_1(t)\right)\bigg] \\
    	&\ = \ 2a_1(t)b_1(t)\left(s_1-s_2\right)\left(1-a_1(t)b_1(t)\right).
        \end{aligned}
    \end{equation}
    Denoting $c_1(t) := a_1(t)b_1(t)$, by solving the ODE above, we obtain
    \[ \ln \frac{c_1(t)}{1-c_1(t)} -\ln \frac{c_1(0)}{1-c_1(0)} \stackrel{\eqref{eq:lower-grad ab-2}}{\ge} 2(s_1-s_2)t. \]
    Further we can rewrite the bound as 
    \begin{equation}
    \label{eq:lower-c-2}
        c_1(t) \geq 1-\frac{1}{\frac{c_1(0)}{1-c_1(0)}e^{2(s_1-s_2)t}+1}.
    \end{equation}
    Then we have $1-a_1(t)b_1(t) = \gO\left(e^{-2(s_1-s_2)t} \right)$. The proof is finished.
\end{proof}

\subsection*{Convergence Rates of $s(t)$: Stage 3}
Similarly, before we start our analysis in Stage 3, we need to handle the minor case $t_2:=\inf \{t: \vu(t)^\top\mZ\vv(t) \ge s(t) \}= +\infty$.
\begin{thm}
    Suppose $N=2$, $a_1(0)b_1(0) > 0$ and $t_2=+\infty$. Then we have
    \[ |s(t)-s_1|=\gO\left(e^{-2s_1 t}\right).  \]
\end{thm}

\begin{proof}
    When $t_2 = +\infty$, we have $\vu(t)^\top\mZ\vv(t) < s(t), \forall t \geq 0$. Thus, we obtain 
    \begin{equation}\label{eq:dot-s-2}
        \dot{s}(t) \stackrel{\eqref{eq:grad-s}}{=} N s(t)^{2-\frac{2}{N}} \left(\vu(t)^\top\mZ\vv(t)-s(t)\right)=2 s(t)\left(\vu(t)^\top\mZ\vv(t)-s(t)\right)\leq 0.
    \end{equation}
    We note that by Theorem \ref{thm:ab}, $s(t) \to s_1$. Thus, we conclude
    \begin{equation}\label{eq:lower-s-3}
        s(t) \stackrel{\eqref{eq:dot-s-2}}{\geq} \lim_{t \to +\infty} s(t) = s_1.
    \end{equation} 
    To obtain the bound of $s(t)$, we notice that
    \begin{equation*}
    \begin{aligned}
        \dot{s}(t) &\stackrel{\eqref{eq:grad-s}}{=} N s(t)^{2-\frac{2}{N}} \left(\sum_{j=1}^d s_j a_j(t)b_j(t)-s(t)\right) \stackrel{\eqref{eq:ineq2}}{\leq} N s(t)^{2-\frac{2}{N}} \left(s_1 -s(t)\right) \\
        &\leq N s_1^{2-\frac{2}{N}} \left(s_1 -s(t)\right)=2 s_1 \left(s_1 -s(t)\right).
    \end{aligned}
    \end{equation*}
    By solving the ODE above, we can obtain the upper bound of the evolution $s(t)$ as
    \begin{equation*}
        d \left(\ln \left(s(t)-s_1\right) \right) / dt \leq -2 s_1 \Rightarrow s(t) \leq s_1 +(s(0)-s_1)e^{-2 s_1t}, i.e., s(t)-s_1=\gO\left(e^{-2s_1 t}\right).
    \end{equation*}
    Finally, noting that $s(t) \stackrel{\eqref{eq:lower-s-3}}{\geq} s_1$, we obtain $|s(t)-s_1| = \gO\left(e^{-2s_1 t}\right)$. The proof is finished.
\end{proof}

Now we turn to the case $t_2<+\infty$. We assume $a_1(0)b_1(0)>0$ and $\dot{s}(0) \geq 0$ for short in Stage 3.

\begin{thm}
	Assume $N=2$, $a_1(0)b_1(0) > 0$, and $\dot{s}(0) \geq 0$. Then we have
	\begin{equation*}
	   |s_1-s(t)|=\gO\left(e^{-\min\{2\left(s_1-s_2\right), 2s(0)\}t} \right).
	\end{equation*}
\end{thm}

\begin{proof}
    By Theorem \ref{thm:monotone}, we have $\dot{s}(t) \geq 0, \forall t\geq 0$. Thus, we can lower bound $s(t) \geq s(0)>0$. Furthermore, we have
    \begin{equation*}
        \begin{aligned}
        \dot{s}(t) &\stackrel{\eqref{eq:grad-s}}{=} N s(t)^{2-\frac{2}{N}} \bigg(\sum_{j=1}^d s_j a_j(t)b_j(t)-s(t)\bigg) \stackrel{\eqref{eq:pre-sab}}{\geq} N s(t)^{2-\frac{2}{N}} \left[ s_1 a_1(t)b_1(t) - s_2 \left(1-|a_1(t)b_1(t)|\right) - s(t)\right] \\
        &\geq  N s(t)^{2-\frac{2}{N}} \left[ (s_1 + s_2) a_1(t)b_1(t)-s_2-s(t)\right] \stackrel{\eqref{eq:lower-c-2}}{\geq} 
        2 s(0) \left[ -(s_1 + s_2)\left(1+A e^{2\left(s_1-s_2\right)t}\right)^{-1} + s_1-s(t)\right],
        \end{aligned}
    \end{equation*}
    where $A=\frac{c_1(0)}{1-c_1(0)}$. By solving the above ODE, we get
    \begin{equation*}
         \frac{d(s(t)e^{2s(0) t})}{dt} \geq 2s(0) e^{2s(0) t}\left[s_1 -(s_1 + s_2)\left(1+A e^{2\left(s_1-s_2\right)t}\right)^{-1} \right].
    \end{equation*}
    Hence, we get
    \begin{align*}
        s(t)e^{2s(0) t}-s(0)& \geq s_1\left(e^{2s(0) t}-1\right) - \int_{0}^t 2s(0)e^{2s(0)x}(s_1 + s_2)\left(1+A e^{2\left(s_1-s_2\right)x}\right)^{-1} dx \\ &\geq  s_1\left(e^{2s(0) t}-1\right) - e^{2s(0)t}\int_{0}^t 2s(0)(s_1 + s_2)\left(1+A e^{2\left(s_1-s_2\right)x}\right)^{-1} dx \\
        &= s_1\left(e^{2s(0) t}-1\right) - \frac{2s(0)(s_1 + s_2)e^{2s(0)t} }{2(s_1-s_2)} \cdot \ln \left(1+A^{-1} e^{-2\left(s_1-s_2\right)t}\right) \\
        &\geq  s_1\left(e^{2s(0) t}-1\right) - \frac{2s(0)(s_1 + s_2)e^{2s(0)t} }{2A(s_1-s_2)} \cdot e^{-2\left(s_1-s_2\right)t}. \\
    \end{align*}
    Therefore, we obtain
    \[ s(t) \geq s_1- (s(0)+s_1) e^{-2s(0) t} - \frac{2s(0)(s_1 + s_2)}{2A\left(s_1-s_2\right)} \cdot e^{-2\left(s_1-s_2\right)t}. \]
    After hiding the constants, we obtain
    \begin{equation*}
       s_1 - s(t) =\gO\left(e^{-\min\{2\left(s_1-s_2\right), 2s(0)\}t} \right).
    \end{equation*}
     We note that by Theorem \ref{thm:ab}, $s(t) \to s_1$. Since $s(t)$ is non-decreasing, we conclude 
     \begin{equation*}
        s(t) \le \lim_{t \to +\infty} s(t) = s_1.
    \end{equation*} 
     Then we obtain $|s_1-s(t)|=\gO\left(e^{-\min\{2\left(s_1-s_2\right), 2s(0)\}t} \right)$. The proof is finished.
\end{proof}


\end{document}